\documentclass[12pt]{article}

\usepackage{times}
\usepackage{authblk}
\usepackage[margin=1in]{geometry}
\usepackage{amsmath, amsthm, amssymb, natbib, graphicx, url, algorithm2e}
\usepackage{url}

\usepackage{breakurl}
\usepackage[colorlinks=true,urlcolor=blue,citecolor=blue,breaklinks]{hyperref}
\usepackage{enumitem,tcolorbox}

\newcounter{Vx}
\newcommand{\itemV}{%
    \addtocounter{Vx}{1}
    \item[V\theVx.]}

\usepackage{graphicx} 
\graphicspath{ {images/} } 

\usepackage{bm}

\usepackage{breakcites}

\usepackage{natbib}
\usepackage{float}
\setcitestyle{%
    authoryear,%
    open={[},close={]},citesep={;},%
    aysep={},yysep={,},%
    notesep={, }}
\let\cite\citep

\newcommand{\R}{\ensuremath{\mathbb{R}}}

\newcommand{\eps}{\varepsilon}

\newcommand{\Z}{\mathcal{Z}}
\newcommand{\boldsigma}{\boldsymbol{\sigma}}
\newcommand{\boldxi}{\boldsymbol{\xi}}

\newcommand{\E}{\mathbb{E}}
\newcommand{\la}{\Lambda}
\newcommand{\s}{\textbf{s}}
\newcommand{\p}{\textbf{p}}
\newcommand{\z}{\textbf{z}}

\newcommand{\ep}{\epsilon}

\newcommand{\x}{\textbf{x}}

\newcommand{\ind}[1]{\ 1\hspace{-2.3mm}{1}_{\left\{#1\right\}}}

\renewcommand{\v}{\ensuremath{\mathbf{v}}}

\newtheorem*{theorem*}{Theorem}
\newtheorem{theorem}{Theorem}
\newtheorem{lemma}[theorem]{Lemma}
\newtheorem{remark}{Remark}

\title{A New Approach to Drifting Games,
Based on Asymptotically Optimal Potentials}
\author{
Zhilei Wang\thanks{zhileiwang92@gmail.com} \qquad Robert V. Kohn\thanks{kohn@cims.nyu.edu; partial support from NSF grant DMS-2009746 is
gratefully acknowledged.}
}
\affil{Courant Institute of Mathematical Sciences, New York University}
\date{}

\usepackage{times}

\begin{document}

\maketitle

\begin{abstract}
We develop a new approach to drifting games, a class of two-person games with many applications to boosting and online learning settings. Our approach involves (a) guessing an asymptotically optimal potential by solving an associated partial differential equation (PDE); then (b) justifying the guess, by proving upper and lower bounds on the final-time loss whose difference scales like a negative power of the number of time steps. The proofs of our potential-based upper bounds are elementary, using little more than Taylor expansion. The proofs of our potential-based lower bounds are also elementary, combining Taylor expansion with probabilistic or combinatorial arguments. Not only is our approach more elementary, but we give new potentials and derive corresponding upper and lower bounds that match each other in the asymptotic regime. 
\end{abstract}


\section{Introduction}

This paper develops a fresh approach to the analysis of some drifting games. Our focus is on
the identification of asymptotically optimal potential-based strategies for some versions of this
repeated two-person game. Our approach involves (a) guessing an asymptotically optimal
potential by solving an associated PDE (which is in general highly nonlinear); then
(b) justifying the guess, by proving upper and lower bounds on the final-time loss whose
difference scales like a negative power of the number of time steps. Our upper bounds
are based on potential-based strategies for the player, and our lower bounds are similarly based
on strategies for the adversary. Their proofs are rather elementary, using
Taylor expansion and the explicit character of the potential.  Most previous work on asymptotically optimal strategies has used potentials obtained by solving a discrete dynamic programming principle, which is complicated and sometimes intractable. Our approach is facilitated by the fact that our potentials are explicit and the arguments are based on basic calculus. Not only is our approach more elementary, but we give new potentials and derive corresponding upper and lower bounds that match each other in the asymptotic regime. In particular, in Section \ref{Sec: continuous shifted action set} we give asymptotically optimal bounds for a drifting game where the adversary could move each chip in a continuous range $[-1,1]$, answering an open question from \cite{Schapire99driftinggames,schapire2001drifting}.

Drifting games are repeated
two-person games involving a {\it player} and an {\it adversary}, whose interaction governs the
positions of $N$ ``chips''. The game is determined by
\begin{enumerate}
\item[(i)] the number of chips $N$;
\item[(ii)] the permitted moves, a subset $\Z$ of the real line;
\item[(iii)] a nonnegative parameter $\delta$, whose role will be revealed in a moment;
\item[(iv)] the number of time steps $T$, and
\item[(v)] the loss function $L : \R \rightarrow \R^+\cup\{0\}$.
\end{enumerate}
When the game begins, all the chips are located at $0$. In each round of the game,
\begin{itemize}
\item the player announces a weight for each chip, i.e., a probability distribution $\p=(p_1, \ldots, p_N)$;
\item the adversary then moves each chip by $z_i$ subject to the restrictions that (a) for each
$i$, the displacement $z_i$ of the $i$th chip belongs to the set $\Z$, and (b) taken together, the
displacements satisfy $\sum_{i=1}^Np_i z_i \geq \delta$.
\end{itemize}
When the game stops (after $T$ time steps), the position $s_i$ of the $i$th chip is the sum of all its moves,
and the player's loss is $\frac{1}{N} \sum_{i=1}^N L(s_i)$. The player's goal is to minimize its loss, and
the adversary's goal is to maximize it.

We shall focus mainly on four versions of this game:
\begin{tcolorbox}
    \begin{enumerate}
        \itemV $Z = \{ \pm 1 \}$, $\delta \geq 0$, and $L(s) = \ind{s \leq 0}$;
        \itemV $Z = [-1,1]$, $\delta \geq 0$, and $L(s) = \ind{s \leq 0}$;
        \itemV $Z = \{ \pm 1 \}$, $\delta=0$, and $L(s) = \ind{s \leq -R}$ for some $R \geq 0$; and
        \itemV $Z = [-1,1]$, $\delta = 0$, and $L(s) = \ind{s \leq -R}$ for some $R \geq 0$.
    \end{enumerate}
\end{tcolorbox}


V1 and V3 are related to classical boosting and \emph{prediction with expert advice} (c.f. \cite{boost-by-majority,Schapire99driftinggames,schapire2001drifting,Boosting_book,Cesa-Bianchi96}); V4 is related to \emph{prediction with continuous experts} and \emph{hedge game} (see e.g.\cite{ContExperts,HedgeGame}). 

V2 is closely related to a ``continuous" boosting game where the weak learners are not binary: they give numbers between $[-1,1]$ for each sample point, which express not only their predictions but also their confidence. We provide bounds for V2 in Section \ref{Sec: continuous shifted action set}, and explain in Section \ref{subsec:continuous boosting game} how these bounds control the training error of the ``continuous" boosting game.

For V1 and V2, we shall assume $\delta \leq 1$ since each $|z_i| \leq 1$. Moreover, we always scale $\delta$ with $T$ so as to keep
\begin{equation} \label{eqn: defn-of-gamma, V1 and V2}
\gamma = \frac{\delta^2 T}{2}
\end{equation}
constant. This choice is required for the condition $\sum p_i z_i \geq \delta$ to be meaningful in the limit
$T \rightarrow \infty$; we shall briefly explain why in Section \ref{subsec:2.2}, and offer a different perspective in Appendix \ref{Sec: heuristic pde derivations}.

Similarly for V3 and V4, we scale $R$ so that 
\begin{equation} \label{eqn: defn-of-gamma, V3 and V4}
\gamma = \frac{R^2}{2T}\tag{1'}
\end{equation}
is a constant. This choice comes from the intuition that $R$ plays the role of $\delta T$ in V1 and V2.


We give a brief review of the literature. Drifting games were first introduced in \cite{Schapire99driftinggames,schapire2001drifting}, 
as an abstraction which generalizes the \emph{majority-vote
game} considered in \cite{boost-by-majority}. Connections to boosting are surveyed in \cite{Boosting_book}. A number of online learning problems can also be studied this way.
For example in \emph{prediction with expert advice} the binomial weights algorithm can be viewed as a
potential-based strategy for the player of V3 \cite{Schapire99driftinggames,schapire2001drifting} and a ``continuous variant" is studied using V4 in \cite{ContExperts}.
A general mechanism for the design of online learning algorithms based on drifting games is
proposed in \cite{HedgeGame}.


It is natural to ask: can we identify optimal strategies for the two players in a drifting game?
An affirmative answer based on dynamic programming was obtained in \cite{Schapire99driftinggames,schapire2001drifting}. In these papers, the player's strategy is given explicitly in terms of a
time-dependent potential that solves a dynamic programming principle. And for V1 and V3, the
associated optimal player strategies are actually ones that had already been considered in earlier work -- the ``boost-by-majority'' and ``binomial weights'' strategy. The adversary's strategy is not given explicitly; rather, its existence is proved by a probabilistic argument, provided that the number of chips is sufficiently large. To deal with a continuous version of \emph{prediction with expert advice}, \cite{ContExperts}
returned to game V4. For V4, the paper identified the potential and the associated optimal strategies, which involves a suitable truncation of the binomial weights algorithm. 

Our work is related to --  but different from -- the developments just summarized. Briefly, we offer a fresh approach to the identification of asymptotically
optimal strategies for the player and the adversary, which can be used even in cases like $\Z = [-1,1]$. In essence, our idea is to (a) guess a potential, by
solving a (nonlinear) PDE that emerges from scaling limit, then (b) show directly, by an argument
based on Taylor expansion combined with the minimax character of the game, that associated potential-based strategies for the player or the adversary are
asymptotically optimal in the limit $T \rightarrow \infty$. This idea is implemented here for the four versions
of the drifting game, and our main results are
\begin{theorem}\label{thm: main thm}
For V1 and V3 (with $\gamma$ defined as in Equation \eqref{eqn: defn-of-gamma, V1 and V2} and \eqref{eqn: defn-of-gamma, V3 and V4}, respectively) and any $\theta\in(0, \frac{2}{3})$, there exists a potential-based player strategy such that
\[
\frac{1}{N}\sum_{i=1}^N L(s_{i}) \leq \frac{1}{\sqrt{\pi}}\int^{\infty}_{\sqrt{\gamma}}e^{-x^2}dx+O(T^{-\frac{\theta}{4}})~; 
\]
when $N$ is sufficiently large (polynomial in $T$), for any player strategy there exists an adversary strategy such that
\[
\frac{1}{N}\sum_{I=1}^N L(s_{i}) \geq \frac{1}{\sqrt{\pi}}\int^{\infty}_{\sqrt{\gamma}}e^{-x^2}dx-O(T^{-\frac{\theta}{4}})~. 
\]

For V2 and V4 (with $\gamma$ defined as in Equation \eqref{eqn: defn-of-gamma, V1 and V2} and \eqref{eqn: defn-of-gamma, V3 and V4}, respectively) and any $\theta\in(0, \frac{2}{3})$, there exists a potential-based player strategy such that
\[
\frac{1}{N}\sum_{I=1}^N L(s_{i}) \leq \frac{2}{\sqrt{\pi}}\int^{\infty}_{\sqrt{\gamma}}e^{-x^2}dx+O(T^{-\frac{\theta}{2}})~; 
\]
when $N$ is sufficiently large (polynomial in $T$), for any player strategy there exists an adversary strategy such that
\[
\frac{1}{N}\sum_{I=1}^N L(s_{i}) \geq \frac{2}{\sqrt{\pi}}\int^{\infty}_{\sqrt{\gamma}}e^{-x^2}dx-O(T^{-\frac{\theta}{4}})~. 
\]
\end{theorem}

Note that as $T \rightarrow \infty$ and $N$ polynomial in $T$, the leading order term in the upper and lower bounds of above games coincide. Moreover, the leading order term for V2 and V4 is twice as the leading order term for V1 and V3, which is not surprising as the adversary has more choices in V2 and V4. The earlier work \cite{ContExperts} observed the same relation between game V3 and V4.

We are not the first to connect drifting games with the solutions of suitable PDEs. Indeed, the paper \cite{BrownBoost} found
an adaptive version of the boost-by-majority algorithm by considering the limit of the \emph{majority-vote game} when $\delta\rightarrow 0$, in other words, when the advantage of each vote over random
guessing decreases to zero while the number of boosting rounds goes to infinity. The paper found a
PDE that corresponds to this limit and named the algorithm Brownboost since the PDE is closely
related to Brownian motion with drift. Subsequently, \cite{FREUND2002113} observed that when
$\delta$ is small the solution of the dynamic programming principle (defined at discrete times using the minimax character of the game) has a particularly simple recursion form. Taking the scaling limit of the recursion formula leads to a PDE. Using the solutions of the PDE with different loss functions, this work
successfully recovered some known boosting algorithms and designed some new ones.
A nice summary can be found in Chapter 14 of \cite{Boosting_book}.

Our potentials are continuous-time limits of the 
discrete potentials found in \cite{Cesa-Bianchi96}, \cite{Schapire99driftinggames,schapire2001drifting} and \cite{ContExperts}. Our Taylor-expansion-based approach has, however, some advantages over the analyses in those papers; in
particular, since the potential is explicit, all its properties are immediately evident and the final loss is easy to characterize.
Similar
applications of Taylor expansion were used in recent papers on \emph{prediction with expert advice}
\cite{pmlr-v107-kobzar20a,pmlr-v125-kobzar20a}.

The character of our potentials plays the central role in our analysis. For V1, our potential has the form 
\begin{equation} \label{separable-potential}
\la(\s,t) = \frac{1}{N} \sum_{i=1}^N f(s_i, t)~.
\end{equation}
The function $f$ is determined by solving a linear heat equation in one space dimension
\begin{align}  \label{eqn: backward heat equation}
\begin{cases}
\partial_t f(s, t)+\frac{1}{2}f''(s, t) = 0\\
f(s, 0) = \ind{s \leq -\delta T}~,
\end{cases}
\end{align}
then introducing suitable shifts in space and time (see Section \ref{Sec: discrete action set}). While this potential is familiar from the literature on boosting, our use of it is different from 
what one finds there: we establish its asymptotic optimality not by finding optimal discrete-time 
potentials then taking a limit, but instead by a rather elementary Taylor-expansion-based argument. 

For V2 our potential still has the separable form \eqref{separable-potential}, but the linear PDE \eqref{eqn: backward heat equation} is replaced by the \emph{nonlinear} PDE 
\begin{align}\label{eqn: backward nonlinear equation}
\begin{cases}
\partial_t f(s, t)+\frac{1}{2}\max(f''(s, t), 0) = 0\\
f(s, 0) = \ind{s \leq -\delta T}~.
\end{cases}
\end{align}

The solution is again explicit, as we explain at the beginning of Section 
\ref{Sec: continuous shifted action set}. Once again, our potential is familiar: indeed, 
it is the scaling limit of the one
found in \cite{ContExperts}, which (as noted earlier) involves a suitable truncation of the binomial weights algorithm. However, our use of this
potential is different from that of \cite{ContExperts}.



The preceding discussion emphasized the use of Taylor expansion to assess the player's strategy by proving an upper bound on the final-time loss. The arguments used
to formulate and assess the adversary's strategy also use our potential and Taylor expansion, but they require some additional arguments to know the existence of a good choice of $\z$ for the adversary. In this area, we adopt methods that are already in the literature. For V1
we use a probabilistic argument similar to that of \cite{schapire2001drifting}. For V2,V3 and V4 a simpler argument is possible, by arguing as in \cite{ContExperts}.
In both cases, it is necessary to assume that the number of chips is large enough.

It is natural to ask \emph{why} the nonlinear PDE \eqref{eqn: backward nonlinear equation} is relevant when $\Z = [-1,1]$, and
to ask more generally how, in other settings, one might use a suitable PDE to guess a good potential (whose validity might then be confirmed using the
methods in this paper). This is addressed in Appendix \ref{Sec: heuristic pde derivations}. 
The discussion is heuristic; however the rest of our paper is entirely rigorous.

This paper is organized as follows: Section \ref{Sec: main ideas} discusses the ideas that
drive our analysis. Sections 
\ref{Sec: discrete action set}, \ref{Sec: continuous shifted action set} provide additional ideas and specific potentials used for games V1 and V2; games V3 and V4 are similar, so their discussion is postponed to Appendix \ref{Sec: discrete action set 2} and \ref{Appendix: continuous action set variant 2}. The careful 
statements of our rigorous results for the
four versions of the drifting game are given in Appendices \ref{Appendix: discrete action set}--\ref{Appendix: continuous action set variant 2} together with detailed proofs. Section 
\ref{Sec: heuristic pde derivations} provides a heuristic derivation of the nonlinear PDE that
conjecturally describes the scaling limit of a fairly general drifting game, while Appendix \ref{Appendix: PDE derivation} justifies a step in our heuristic identification 
of the PDE under a reasonable assumption.

\section{The Main Ideas} \label{Sec: main ideas}
We give more details of drifting games in subsection \ref{subsec:2.1}. Subsection \ref{subsec:2.2} explains why $\delta$ should scale with $T$ by the law
\eqref{eqn: defn-of-gamma, V1 and V2}. Then in Sections \ref{subsec:key ideas}--\ref{subsec:broader classes} we discuss the key ideas that lie behind our analysis. The section closes with a brief summary of 
some notational conventions.

\subsection{The Drifting Game and its Minimax Loss} \label{subsec:2.1}

It is convenient to let the game
end at time $0$; therefore the game starts at time $-T$ and its final round occurs
at time $-1$. The player's choice of $\p = (p_1,\ldots,p_N)$ at time $t$ will be
called $\p_t = (p_{t,1},\ldots,p_{t,N})$, and the adversary's choice of
$\z = (z_1,\ldots,z_N)$ at time $t$ will be called $\z_t = (z_{t,1}, \ldots, z_{t,N})$.
Since each chip is initially located at $0$, the chips' positions at time $t$ satisfy
$s_{-T,i}=0$ and $s_{t+1,i}=s_{t,i}+z_{t,i}$, and their final positions are $s_{0,i}$;
as usual, we shall write $\s_t = (s_{t,1}, \ldots, s_{t,N})$.

We recall that at time $t$, the adversary may choose any $\z_t$ such that $z_{t,i} \in \Z$
for each $i$ and $\p_t \cdot \z_t \geq \delta$.  It is convenient to give the admissible
set a name; we therefore define
$$
S_\delta (\p) = \{ \z\in {\Z}^N | \p \cdot \z \geq \delta\} .
$$

We define $\la^d_\delta(\s,t)$ (the superscript $d$ stands for ``discrete'') to be the player's final-time loss (assuming
optimal play by both parties), if the chips' locations are $\s$ at time $t$. It
is characterized by the dynamic programming principle
\begin{equation}\label{eqn:DPP for original game}
\la^d_\delta (\s,t)=\min_\p\max_{\z\in S_\delta(\p)}\la^d_\delta (\s+\z,t+1) \quad
\mbox{for $t \leq -1$}
\end{equation}
combined with the final-time condition 
$$
\la^d_\delta(\s,0)=\frac{1}{N}\sum_{i=1}^N L(s_i)~.
$$

The dynamic programming principle \eqref{eqn:DPP for original game} defines
$\la^d_\delta (\s,t)$ for all $\s \in \R^N$ and all negative integer times $t$.
Our goal is to estimate the final-time loss when the chips are initially at $0$, i.e. $\la^d_\delta(\bm{0},-T)$, and this is the player's \emph{minimax loss}. Unpacking the dynamic programming principle, it is
$$
\la^d_\delta(0,-T) = \min_{\p_{-T}}\ \max_{\z_{-T}\in S_\delta(\p_{-T})} \ldots
\min_{\p_{-1}}\ \max_{\z_{-1} \in S_\delta(\p_{-1})}
\frac{1}{N}\sum_{i=1}^NL \bigl(\sum_{t=-T}^{-1}z_{t,i}\bigr).
$$

\subsection{The Dependence of \texorpdfstring{$\delta$ on $T$}{interval}} \label{subsec:2.2}

It is already familiar from \cite{FREUND2002113,Boosting_book} that when focusing on the asymptotic behavior as
$T \rightarrow \infty$, the parameter $\delta$ should depend on $T$ by \eqref{eqn: defn-of-gamma, V1 and V2}. The argument
there involves taking a scaling limit of the game (in roughly the same way that Brownian motion arises
as the scaling limit of random walk on a lattice). We shall consider such a scaling limit in
Section \ref{Sec: heuristic pde derivations}, where we discuss how our ideas might extend to
a more general class of drifting games.

It is, however, possible to understand the dependence of $\delta$ on $T$ quite simply, as follows. As already noted
in the Introduction, our potentials are built from solutions of Equation \eqref{eqn: backward heat equation} or
\eqref{eqn: backward nonlinear equation}, and in each case the solution has the form
$f(s,t) = G\left(\left(s+\delta T\right)/\sqrt{|t|}\right)$. For $f(s,-T)$ to have a nontrivial limit as $T \rightarrow \infty$, $\delta T / \sqrt{T}= \delta \sqrt{T}$ should be held constant. Our law \eqref{eqn: defn-of-gamma, V1 and V2} is a restatement of this condition.


\subsection{The Role of Taylor Expansion, and Relevance of the PDEs} \label{subsec:key ideas}

The role of Taylor expansion is the following: to show that a potential $\la (\s,t)$ is
approximately equal to $\la^d_\delta$, we need only check that it is close at the final time and that
it approximately satisfies the dynamic programming principle. The quality of the approximation can be
estimated by adding up the final-time error and the sum of all the approximation errors at times
$-T,-T+1, \ldots -1$.


Note that Taylor expansion requires smoothness, whereas the
final-time data for these PDEs (our loss function) is discontinuous, so we cannot use the solution of the
PDE directly. Rather, we use a \emph{shifted} version of it; specifically, rather
than use the solution $f$ of Equation \eqref{eqn: backward heat equation} or
\eqref{eqn: backward nonlinear equation} our potentials have the separable form
\eqref{separable-potential} with $f$ replaced by $\tilde{f}(s,t) = f(s + \beta, t - \tau) + c$,
where $\beta$, $\tau$, and $ c$ are $T$-dependent constants. The constant $\tau > 0$ is chosen so that
$\tilde{f}$ is sufficiently smooth for $t = 0$. The constants $\beta$ and $c$ are chosen so that $\tilde{f}(s,0) \geq L(s)$
for an upper-bound potential (used to identify a good strategy for the player), or so that
$\tilde{f}(s,0) \leq L(s)$ for a lower-bound potential (used to identify a good strategy for the adversary).

We turn now to a more quantitative explanation of how Taylor expansion will be used. 
For any smooth potential $\la(\s,t)$, Taylor's theorem gives
\begin{align} \label{eqn:Taylor expansion}
\la(\s+\z,t+1)-\la(\s,t) = & [\la(\s+\z,t+1) - \la(\s, t+1)] + [\la(\s,t+1) - \la(\s,t)] \nonumber \\
 = & \nabla\la(\s, t+1)\cdot\z +\left(\partial_t\la(\s,t+1)+\frac{1}{2}\z^\top D^2\la(\s, t+1)\z\right)+ E_t
\end{align}
where $E_t$ represents the ``error'' due to the chosen truncation of the Taylor expansion. To illustrate the main idea, we ignore $E_t$ and the smoothness issue of $\la$ for now.

The dynamic programming principle \eqref{eqn:DPP for original game} says that when $\la$ is replaced
by $\la^d_\delta$, the min-max of the left hand side is $0$. Therefore we would like
the min-max of the right side to be nearly $0$:
$$
\min_{\p} \max_{\z \in S_\delta (\p)} \nabla\la(\s, t+1)\cdot\z +
\left(\partial_t\la(\s,t+1)+\frac{1}{2}\z^\top D^2\la(\s, t+1)\z\right) \approx 0.
$$
We emphasize that to prove such a statement, one must identify good strategies for both the player and
the adversary. The player's strategy is a choice of $\p$ such that for every ${\z \in S_\delta (\p)}$,
$ \nabla\la(\s, t+1)\cdot\z + \left(\partial_t\la(\s,t+1)+\frac{1}{2}\z^\top D^2\la(\s, t+1)\z\right) \leq 0$. The adversary's strategy is a way of choosing $\z \in S_\delta (\p)$ (given $\p$),
such that
$ \nabla\la(\s, t+1)\cdot\z + \left(\partial_t\la(\s,t+1)+\frac{1}{2}\z^\top D^2\la(\s, t+1)\z\right) \geq 0$.

The situation is simplest to understand when $\delta = 0 $ and $\Z = \{ \pm 1 \}$.
Then the player (who chooses $\p$ and wants to minimize) can make the first-order term
non-positive by choosing $\p$ proportional to $-\nabla \la$. (We use here that our potential has
$\partial \la / \partial s_i \leq 0$ for any $i \in [N]$.) It is natural to guess that the adversary (who
chooses $\z$ after hearing $\p$, and wants to maximize)
can choose $\z$ so that the first-order term is close enough to $0$ to treat it as an ``error term.''
If the potential has the form
\begin{equation} \label{separable-potential-bis}
\la(\s ,t) = \frac{1}{N} \sum_{i=1}^N \tilde{f}(s_i, t)
\end{equation}
for some scalar-valued function $\tilde{f}(s,t)$ then
$$
\z^\top D^2\la(\s, t+1) \z = \frac{1}{N} \sum_{i=1}^N \tilde{f}''(s_i,t)
$$
is independent of $\z$ (since each $z_i = \pm 1$). So the min-max vanishes (modulo error terms)
when $\tilde{f}_t + \frac{1}{2} \tilde{f}'' = 0$.

The situation is only a little different when $\delta = 0$ and $\Z = [-1,1]$ (provided we ignore the
non-smoothness of the potential). The player can still make the first-order term nonpositive by
choosing $\p$ proportional to $- \nabla \la$, and the adversary can still make the first-order term
close enough to zero that it becomes an ``error term.'' But when $\Z = [-1, 1]$, the second-order
Taylor expansion term is no longer independent of $\z$, and
$$
\max_{\z \in [-1,1]^N} \z^\top D^2\la(\s, t+1) \z = \frac{1}{N} \sum_{i=1}^N \max \{\tilde{f}''(s_i,t+1),0\},
$$
since the optimal $z_i$ is $\pm 1$ when $\tilde{f}'' \geq 0$ and $0$ when $\tilde{f}'' \leq 0$.
This is the origin of the PDE \eqref{eqn: backward nonlinear equation}.

\subsection{Comments on Broader Classes of Drifting Games} \label{subsec:broader classes}

We have thus far discussed a few specific examples of drifting games. One might wonder whether PDEs can
be used to determine the minimax loss in the limit $T \rightarrow \infty$ for a more general class of drifting games. While we have no rigorous results of this kind, Appendix \ref{Sec: heuristic pde derivations} 
offers a suggestion. The arguments there use
a scaled version of the game. Not surprisingly, when specialized to the cases considered in this section, they
reduce to the heuristic derivations of the PDEs that we have given in this section (see Appendix \ref{subsec:7.3}).

As noted earlier, our lower bounds require that $N$ be sufficiently large. It is natural to ask whether this is just a limitation of our method. We think not; rather, the situation is fundamentally
different when $N$ is small. Indeed, the arguments in Appendix \ref{subsec:7.2} consider what happens when $T \rightarrow \infty$ with $N$ held fixed. When $\Z = [-1,1]$ and $\delta = 0$, they suggest that the PDE associated with the (optimal) player potential $\la(\s, t)$ should be
$$
\partial_t \la +
\frac{1}{2} \, \underset{\nabla\la \perp \z, \, \z \in [-1,1]^N}
{\max} \z^\top D^2\la \z = 0
$$
(see Equation \eqref{eqn: PDE for drifting game}). The potential we use
for our rigorous results solves a similar but different equation, in
which the max is over all $\z \in [-1, 1]^N$. It seems that, when $N$ is large,
the constraint $\nabla
\la \perp \z$ is unimportant as there always exists $\z$ s.t. $\nabla \la \cdot \z \sim O(\frac{1}{N})$, whereas when $N$ is small the constraint is
important.

\subsection{Notation} \label{subsec:2.7}
We introduce some notation that are used throughout this paper. 
For a function $f(s,t)$ where $s$ is the spatial variable and $t$ is the time variable, we use $\partial_t,\partial_{tt},\ldots$ to represent time derivatives and $f',f'',f^{(3)},f^{(4)}\ldots$ to represent spatial derivatives. For functions with more than one spatial variable, $\nabla$ and $D^2$ are used to represent the gradient and Hessian of the spatial variables. Bold letters, such as $\s$, $\z$ and $\p$, are vectors
in $\mathbb{R}^N$, and normal letters like $s$ and $t$ are scalars. Also, $\bm{1}$ and $\bm{0}$
are vectors of all 1's and all 0's respectively. For the sake of brevity we sometimes omit the time subscript and write $s_{t,i}$, $z_{t,i}$ and $p_{t,i}$ as $s_i$,
$z_i$ and $p_i$ respectively. We use $C$ to represent absolute constants and big $O$ notation $f=O(g)$ means $f\leq C g$ for some $C$. Lastly, the ceiling function
$\lceil x \rceil$ is the smallest integer greater than or equal to $x$.


\section{Game V1}\label{Sec: discrete action set}
In this section we consider game V1: $\delta \geq 0, \mathcal{Z} = \{\pm 1\}$, and $L(s)=\ind{s \leq 0}$. The ideas here extend straightforwardly to game V3 (see Appendix \ref{Sec: discrete action set 2}). This game is related to classical boosting setting where the weak learners make predictions from $\{\pm 1\}$. We give new player and adversary strategies for this game (thus also for the classical boosting setting), and give (matching) upper and lower bounds which can be seen as the limit of the discrete bounds given in \cite{Schapire99driftinggames,schapire2001drifting} as $T\rightarrow\infty$.  
\subsection{Additional Ideas in Game V1}\label{subsec:additional ideas in V1}
Replacing $z_i$ by $z_i - \delta$, it is equivalent to consider situation when
$\Z = \{ -1-\delta, 1-\delta \}$, $L(s) = \ind{\s \leq - \delta T}$, and the adversary's global constraint is
$\p \cdot \z \geq 0$. The player's strategy should, as explained above, be to choose $\p$ proportional
to $-\nabla \la$. The adversary needs to choose $z_i \in \{ -1-\delta, 1-\delta \}$ such that $\p \cdot \z$
is very near $0$; our proof that this is possible is probabilistic, adopting an argument from
\cite{Schapire99driftinggames, schapire2001drifting}. However, when $\delta > 0$ the second-order Taylor expansion
term is no longer independent of the choice of $z_i \in \{-1-\delta, 1 - \delta\}$. Fortunately, $\delta$ is
small when $T$ is large, as a result of the scaling \eqref{eqn: defn-of-gamma, V1 and V2}. So if $\tilde{f}$ solves
$\tilde{f}_t + \frac{1}{2} \tilde{f}'' = 0$, we can apply the argument sketched in Section \ref{subsec:key ideas}, estimating
$$
\z^\top D^2\la(\s, t+1) \z \sim \frac{1}{N} \sum_{i=1}^N \tilde{f}''(s_i,t) +
\delta \max_{s \in \R} |\tilde{f}''(s,t)|. 
$$
The second term is treated as an ``error term'' (alongside the errors associated with truncation of the Taylor
expansion and nonzero $\nabla \la \cdot \z$).

We work on this equivalent version and build our potential using the solution of Equation \eqref{eqn: backward heat equation}. When $\delta=0$ the explicit solution
\footnote{We can see from Figure \ref{fig:Solution} that $g(\cdot,t)-1/2$ is a strictly decreasing odd function. Also $0\leq g\leq 1$ and $g$ is concave when $s<0$ and convex when $s>0$; moreover as $t\rightarrow -\infty$, $g(s,\cdot)$ is decreasing for $s<0$ and increasing for $s>0$, and $\lim_{t\rightarrow-\infty}g(s,t)=1/2$ for all $s$.} 
is
\[
g(s,t) = \frac{1}{\sqrt{\pi}}\int_{s/\sqrt{-2t}}^\infty e^{-x^2} dx~.
\]
\begin{figure}[H]
    \centering
    \begin{minipage}[t]{0.48\textwidth}
    \centering
    \includegraphics[width = 6cm]{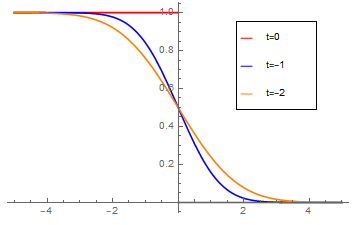}
    \caption{Solution $g$ at $t=0,-1,-2$}
    \label{fig:Solution}
    \end{minipage}
    \begin{minipage}[t]{0.48\textwidth}
    \centering
    \includegraphics[width = 6cm]{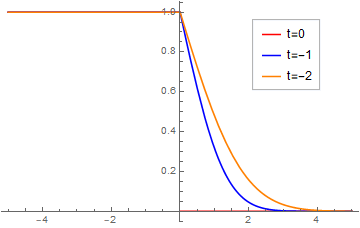}
    \caption{Solution $\tilde{g}$ at $t=0,-1,-2$}
    \label{fig:solution2}
    \end{minipage}
\end{figure}

For Equation \eqref{eqn: backward heat equation}, the solution is
\[
f(s,t)=g(s+\delta T,t)~,
\] 
which serves as the building block of the upper and lower bound potentials.

\subsection{Potentials}\label{subsec:potentials of V1}
As mentioned in \ref{subsec:key ideas}, we do not use $f$ directly; rather we use a shifted version of $f$ to build the upper bound potential $\la_U$:
\begin{align*}
    \la_U(\s, t) = \frac{1}{N}\sum_{i=1}^N f(s_i - \beta, t - \tau) + 1 - f(-\delta T-\beta, -\tau)~,
\end{align*}
with $\beta=\tau=T^{\frac{\theta}{2}}$ for any $\theta \in (0, \frac{2}{3})$. Note that $\la_U$ is a strictly decreasing function on all the spatial coordinates, which implies that $\p_t\sim-\nabla\la_U(\s, t+1)$ is a valid probability distribution. The upper bound corresponding to the player strategy is given in Theorem \ref{thm:upper bound thm 1} (see Appendix \ref{Appendix: discrete action set}).

We use a different shifted version of $f$ to construct the lower bound potential $\la_L$:
\[
\la_L(\s, t) = \frac{1}{N} \sum_{i=1}^N f(s_i + \beta, t - \tau) - f(-\delta T + \beta, -\tau)~,
\]
with $\beta=\tau=T^{\frac{\theta}{2}}$ for any $\theta \in (0, \frac{2}{3})$. When $N$ is at least $O(T^{2+(\theta/2)}\log T)$, the adversary strategy and corresponding lower bound using this potential is given in Theorem \ref{thm: lower bound thm 1} (see Appendix \ref{Appendix: discrete action set}).

\section{Game V2}\label{Sec: continuous shifted action set}
In this section we consider game V2: $\delta \geq 0, \mathcal{Z} = [-1,1]$, and $L(s)=\ind{s \leq 0}$. The ideas here extend straightforwardly to game V4 (see Appendix \ref{Appendix: continuous action set variant 2}). This game is related to a continuous variant of boosting game. In \cite{Schapire99driftinggames,schapire2001drifting} they build this connection and also gave some numerical results for both V2 and the associated \emph{continuous boosting game}, however no characterization of the bound was given. We give new player and adversary strategies for this game (thus also for the \emph{continuous boosting game}) and give upper and lower bounds that match when $T \rightarrow \infty$. 

\subsection{Additional Ideas in Game V2}
\label{subsec:2.5}
As in Section \ref{Sec: discrete action set}, it convenient to replace $z_i$ by $z_i - \delta$, which leads us to consider the drifting game
in which $\Z = [-1-\delta, 1-\delta]$, $L(s) = \ind{s \leq - \delta T}$, and the adversary's global
constraint is $\p \cdot \z \geq 0$. Our potential has the separable form \eqref{separable-potential-bis}, but $\tilde{f} = f(s + \beta, t - \tau) + c$ where $f$ solves a nonlinear PDE like
\eqref{eqn: backward nonlinear equation}. In particular, $\tilde{f}(s,t)$ is constant (independent of
both $s$ and $t$) when $s$ lies below a critical value, while $\tilde{f}$ is convex and smooth when $s$ lies above the critical value.

The player's strategy is always the same: $\p$ should be proportional to $- \nabla \la$. Note that when $s_i$ is
below the critical value this gives $p_i = 0$. The adversary's strategy must, as usual, choose $z_i$ such that $\p \cdot \z $ is close enough to $0$ so $\nabla \la \cdot \z$ can be treated as an error term. We find it convenient to limit the adversary's strategy to $z_i$ taken only from $\{0, \pm(1 - \delta)\}$. In fact, our adversary chooses $z_i = 0$ when $s_i$ is at or below the
critical value, and $z_i \in \{\pm(1 - \delta)\}$ when $s_i$ is above the critical value. Since the two nonzero
possibilities are symmetric, the existence of such $\z$ such that $\p \cdot \z$ is nearly $0$ can be proved using
a combinatorial argument previously used in \cite{ContExperts}. (This is simpler than the probabilistic argument
of \cite{Schapire99driftinggames, schapire2001drifting}, and for a given error it requires a smaller number of
chips.)

Since $\tilde{f}(s,t)$ is only continuous at the critical value of $s$, our use of Taylor expansion
needs to be re-examined. When $s_i$ is well below the critical value, $\tilde{f}$ is locally constant in both space
and time so Taylor expansion is not needed. When $s_i$ is well above the critical value, our Taylor-expansion-based
arguments are applicable. When $s_i$ is close enough to the critical value that $s_i$ and $s_i + z_i$ can be
on opposite sides of it, special attention is needed. Fortunately, the required inequalities are available
in this situation, by combining Taylor expansion \emph{restricted to $s$ greater than the critical value} with
the monotonicity of $\tilde{f}$.

We build our potential using the solution of Equation \eqref{eqn: backward nonlinear equation}. When $\delta=0$, we construct a piece-wise smooth solution
\footnote{From figure \ref{fig:solution2} we can see that on the right of origin $\tilde{g}(\cdot, t)$ is convex, while on the left $\tilde{g}(\cdot, t)=1$ is a constant. The potential $\tilde{g}$ is continuous at origin, but not differentiable. }
\begin{align*}
    \tilde{g}(s,t) =
    \begin{cases}
        2g(s,t)= \frac{2}{\sqrt{\pi}}\int_{s/\sqrt{-2t}}^\infty e^{-x^2} dx &s>0\\
        1 &s\leq 0~,
    \end{cases}
\end{align*}

For Equation \eqref{eqn: backward nonlinear equation} with $\delta>0$, the solution can be written as 
\[
f(s,t)=\tilde{g}(s+\delta T,t)~,
\] 
which serves as the building block of the upper and lower bound potentials.


\subsection{Potentials}\label{subsec:potentials of V2}
We use a time-shifted version as the upper bound potential
\[
\la_U(\s,t) = \frac{1}{N}\sum_{i=1}^N f(s_i, t-\tau)~,
\]
where $\tau=T^\theta$ for any $\theta \in (0, \frac{2}{3})$. $\la_U$ is a decreasing function in all the spatial variables. 

The player imposes the following distribution: if at least one chip is on the right of $-\delta T$,
\begin{align}\label{stgy: continuous action set player strat}
    p_{i} \sim
    \begin{cases}
        0&s_i \leq -\delta T\\
        -f'(s_i, t + 1 - \tau)&s_i > -\delta T~;
    \end{cases}
\end{align}
otherwise if all the chips are on the left of $-\delta T$, any probability distribution is fine. With a more refined analysis compared to game V1 we get Theorem \ref{thm:upper bound thm 2} (see Appendix \ref{Appendix: continuous action set variant 2}).

The lower bound potential $\la_L$ is defined as,
\[
\la_L(\s, t) = \frac{1}{N} \sum_{i=1}^N f(s_i + \beta, t - \tau) - f(-\delta T+\beta, -\tau)~,
\]
with $\beta=\tau=\lceil \frac{\delta T + T^{\frac{\theta}{2}}}{1-\delta} \rceil (1-\delta) -\delta T$ for any $\theta \in (0, \frac{2}{3})$. When $N \geq T^\frac{\theta+2}{4}$, the adversary strategy and corresponding lower bound using this potential is given in Theorem \ref{thm: lower bound thm 2} (see Appendix \ref{Appendix: continuous action set variant 2}).






\subsection{Application to Continuous Boosting Game}\label{subsec:continuous boosting game}
In the classical boosting setting of binary classification, we are given a set of $N$ training examples
$\{(\x_i, y_i)\}_{i \in [N]}$ where $x_i \in \mathcal{X}$ is an example and $y_i \in \{-1, 1\}$ is its label. A boosting algorithm proceeds for $T$ rounds ($-T, \ldots, -1$). At each round, a distribution $\p_t$ over the examples is computed by the player and the adversary returns a “weak” hypothesis $h_t : \mathcal{X} \rightarrow \{-1, 1\}$ with a guaranteed small edge, that is, 
\[
\sum_{i:y_i=h_t(\x_i)} p_{t,i} - \sum_{j:y_j\neq h_t(\x_j)} p_{t,j} \geq \delta.
\]
At the end, the final “strong” hypothesis is a majority vote of all $h_t$'s and it is expected to have low training error and potentially low generalization error.

We consider the setting where the weak hypothesis $h_t$ can predict any value between $[-1,1]$; here the sign is interpreted as a prediction, while the absolute value expresses the confidence level of this prediction. In particular, a weaker learner could choose to abstain from giving any prediction at all by predicting 0. The weak hypothesis $h_t$ satisfies
\begin{align*}\label{eq: *}
    \sum_{i:y_i h_t(\x_i)>0} p_{t,i}|h_t(\x_i)| - \sum_{j:y_j h_t(\x_j) \leq 0} p_{t,j}|h_t(\x_j)| \geq \delta. \tag{*}
\end{align*}

This version of the boosting game is equivalent to game V2. More specifically, each sample-label pair $(\x_i, y_i)$ is a ``chip" and after the player poses an probability distribution $\p_t$ on the N chips, the weak hypothesis $h_t$ given by the adversary determines the movement of chip $i$ by $z_{t,i} = y_ih_t(\x_i)\in[-1,1]$, and (\ref{eq: *}) enforces that $\p_t \cdot \z_t \geq \delta$. The final hypothesis is a majority vote of the weak learners and the training error is
$$
\frac{1}{N} \sum_{i=1}^N \ind{\sum_{t=-T}^{-1} y_ih_t(\x_i) \leq 0} = 
\frac{1}{N} \sum_{i=1}^N \ind{s_{0, i} \leq 0}~,
$$
which is exactly the final loss of drifting game V2.

Therefore, we have the following theorem for the \emph{continuous boosting game}.
\begin{theorem}\label{thm: continuous boosting game}
For the \emph{continuous boosting game} that lasts $T$ rounds with edge $\frac{\delta}{2}$ and any $\theta \in (0, \frac{2}{3})$, the player using the probability distribution described in \eqref{stgy: continuous action set player strat} achieves a training error of at most 
\[
\frac{1}{N} \sum_{i=1}^N \ind{\sum_{t=-T}^{-1} y_ih_t(\x_i) \leq 0} \leq \frac{2}{\sqrt{\pi}}\int^{\infty}_{\sqrt{\gamma}}e^{-x^2}dx+O(T^{-\frac{\theta}{2}})~. 
\]
When the number of samples $N \geq T^{\frac{\theta+2}{4}}$, and assuming the player's final hypothesis is a majority vote of the weak hypothesis, for any such player there exists an adversary strategy such that the training error is at least
\[
\frac{1}{N} \sum_{i=1}^N \ind{\sum_{t=-T}^{-1} y_ih_t(\x_i) \leq 0} \geq \frac{2}{\sqrt{\pi}}\int^{\infty}_{\sqrt{\gamma}}e^{-x^2}dx-O(T^{-\frac{\theta}{4}})~. 
\]
    
\end{theorem}

\begin{remark}
    The player's strategy described by (\ref{stgy: continuous action set player strat}) assigns zero weight for chips which are far from zero; this implies that the boosting algorithm assigns zero weight to samples with large negative margins. This potentially benefits the running time for the boosting algorithm, since the weak hypothesis is trained on fewer samples. 
\end{remark}

\section{Conclusions}\label{Sec: conclusion}
We have developed a PDE approach to some versions of \emph{drifting games}. Using solutions of PDEs as potentials, we give asymptotically optimal strategies for both the player and the adversary, together with upper and lower bounds of the final loss that match when $T\rightarrow\infty$. These new strategies can also be applied to boosting and online learning problems like \emph{prediction with expert advice}, and the upper and lower bounds are also valid in these games. In particular, we obtain an asymptotically sharp loss bound for V2 (thus also for the associated \emph{continuous boosting game}, assuming the player use a majority vote of the weak learners), answering one of the open problems in \cite{Schapire99driftinggames,schapire2001drifting}. We believe the PDE approach might be useful to characterize the loss when $T$ is large in some other cases where the discrete dynamic programming principle has no closed form solution.

\bibliographystyle{plainnat}
\bibliography{reference}

\appendix

\section{Theorems and Proofs for V1}\label{Appendix: discrete action set}
First we control the derivatives of the solution of Equation \eqref{eqn: backward heat equation} by the following lemma.

\begin{lemma}\label{lma: Control of higher order derivatives of g}
There exists constant $C$ s.t. $\|g'(\cdot,t)\|_{L^\infty}\leq\frac{C}{\sqrt{|t|}}$, $\|g''(\cdot,t)\|_{L^\infty}\leq\frac{C}{|t|}$, $\| g^{(3)}(\cdot,t)\|_{L^\infty}\leq\frac{C}{|t|^\frac{3}{2}}$, and $\|\partial_{tt}g(\cdot,t)\|_{L^\infty}\leq \frac{C}{|t|^2}$.
\end{lemma}

\begin{proof}[Proof of lemma \ref{lma: Control of higher order derivatives of g}]
Doing some algebra we have
\begin{align*}
    g'(s,t)=&-\frac{1}{\sqrt{-2\pi t}}e^{\frac{s^2}{2t}}\\
    g''(s,t)=&-2\partial_t g(s,t)=-\frac{1}{t\sqrt{\pi}}e^{\frac{s^2}{2t}}\frac{s}{\sqrt{-2t}}\\
    g^{(3)}(s,t)=&-\frac{2}{t\sqrt{-2\pi t}}e^{\frac{s^2}{2t}}(1/2+\frac{s^2}{2t})\\
    \partial_{tt}g(s,t)=&\frac{1}{4}g^{(4)}(s,t)=-\frac{1}{2t^2}e^{\frac{s^2}{2t}}\frac{s}{\sqrt{-2\pi t}}(3/2+\frac{s^2}{2t});
\end{align*}
thus there exists a constant $C$ such that $\|g'(\cdot,t)\|_{L^\infty}\leq\frac{C}{\sqrt{|t|}}$, $\|g''(\cdot,t)\|_{L^\infty}\leq\frac{C}{|t|}$, $\| g^{(3)}(\cdot,t)\|_{L^\infty}\leq\frac{C}{|t|^\frac{3}{2}}$, and  $\|\partial_{tt}g(\cdot,t)\|_{L^\infty}\leq\frac{C}{|t|^2}$.
\end{proof}

Next we prove the upper bound
\begin{theorem}\label{thm:upper bound thm 1}
For any $\theta\in(0,\frac{2}{3})$ the player strategy proportional to the negative gradient of $\la_U$ with $\beta=\tau=T^\frac{\theta}{2}$ satisfies
\[
\frac{1}{N}\sum_{i=1}^N L(s_{0,i})\leq
    \frac{1}{\sqrt{\pi}}\int^{\infty}_{\sqrt{\gamma}}e^{-x^2}dx+O(T^{-\frac{\theta}{4}})~,
\]
for any adversary strategy $\mathcal{A}$. 
\end{theorem}

\begin{proof}[Proof of Theorem \ref{thm:upper bound thm 1}]
For the sake of simplicity, at any fixed time step $t$, we omit $t$ in the subscripts of $\s,\p,\z,E$ (which stands for the Taylor expansion error). We write the increment of $\la_U$ from $t$ to $t+1$ as
\begin{align*}
    &\la_U(\s+\z, t+1)-\la_U(\s, t)\\
    =& \frac{1}{N}\sum_{i=1}^N f(s_i-\beta+z_i,t+1-\tau) - f(s_i-\beta,t-\tau)\\
    =& \frac{1}{N}\sum_{i=1}^N f(s_i-\beta+z_i,t+1-\tau) - f(s_i-\beta,t+1-\tau) + f(s_i-\beta,t+1-\tau) - f(s_i-\beta,t-\tau)\\
    =& \frac{1}{N}\sum_{i=1}^N f'(s_i-\beta ,t+1-\tau)z_i + \frac{1}{2} f''(s_i-\beta ,t+1-\tau)z_i^2 + \partial_tf(s_i-\beta, t+1-\tau)+ \frac{1}{N}\sum_{i=1}^N E_i~.
\end{align*}
The error term 
\begin{align*}\label{eqn: error term}
    E_i = \frac{1}{6}f^{(3)}(\tilde{\s}_i,t+1-\tau)z_i^3-\frac{1}{2}\partial_{tt}f(s_i-\beta,\tilde{t}_i)~,
\end{align*}
where $\tilde{\s}_i$ is between $s_i-\beta$, $s_i-\beta+z_i$ and $\tilde{t}_i$ is between $t-\tau$ and $t+1-\tau$. 
Using Lemma \ref{lma: Control of higher order derivatives of g} and the fact that $|z_i| \leq 1+\delta$, we can bound $\frac{1}{N}\sum_{i=1}^N E_i$ by 
\[
|\frac{1}{N}\sum_{i=1}^N E_i| \leq \sup_{i\in[N]}|E_i| \leq \frac{C}{|t+1-\tau|^\frac{3}{2}}+\frac{C}{|t+1-\tau|^2}~,
\]
where $C$ is some absolute constant.

Recall that the player set $p_i\sim-f'(s_i-\beta ,t+1-\tau)$ and since $\p\cdot\z \geq 0$, we have
\[
\sum_{i=1}^N f'(s_i-\beta ,t+1-\tau)z_i \leq 0~.
\]
Moreover, as $f$ satisfies Equation \eqref{eqn: backward heat equation}, we have
\begin{align*}
   &\frac{1}{N}\sum_{i=1}^N f'(s_i-\beta ,t+1-\tau)z_i + \frac{1}{2} f''(s_i-\beta ,t+1-\tau)z_i^2 + \partial_tf(s_i-\beta, t+1-\tau)\\
   \leq& \frac{1}{N}\sum_{i=1}^N \partial_tf(s_i-\beta, t+1-\tau) + \frac{1}{2} f''(s_i-\beta ,t+1-\tau)\\
   &+ \frac{1}{N}\sum_{i=1}^N f''(s_i-\beta ,t+1-\tau)\frac{z_i^2-1}{2}\\
   \leq& \frac{3\delta}{2} \sup_{i\in[N]}|f''(s_i-\beta ,t+1-\tau)|\\
   \leq& \frac{C\sqrt{\gamma}}{|t+1-\tau|\sqrt{T}}~.
\end{align*}
The second inequality used the fact that $\delta\in [0,1]$ thus $\delta^2 \leq \delta$, and the last inequality used Lemma \ref{lma: Control of higher order derivatives of g} and the definition of $\gamma$ in Equation \eqref{eqn: defn-of-gamma, V1 and V2}.

Combining the above analysis together we now add up the increment of $\la_U$ from $-T$ to $-1$,
\begin{align*}
    \frac{1}{N}\sum_{i=1}^N L(s_{0,i})\leq\la_U(\s_0,0)&=\la_U(\bm{0}, -T)+\sum_{t=-T}^{-1}\la_U(\s_{t+1}, t+1)-\la_U(\s_t, t)\\
    &\leq\la_U(\bm{0}, -T)+C\sum_{t=-T+1-\tau}^{-\tau}\frac{1}{|t|^2} + \frac{1}{|t|^\frac{3}{2}} + \frac{\sqrt{\gamma}}{|t|\sqrt{T}}\\
    &=\la_U(\bm{0}, -T)+O\left(\frac{1}{\sqrt{\tau}}+\frac{1}{\tau}+\frac{\sqrt{\gamma}\log\tau}{\sqrt{T}}\right)~.
\end{align*}
Note that
\begin{align*}
    \la_U(\bm{0}, -T)&=f(-\beta,-T-\tau)+1-f(-\delta T-\beta, -\tau)\\
    &=1-\frac{1}{\sqrt{\pi}}\int^{\frac{\beta}{\sqrt{2\tau}}}_{\frac{\beta-\delta T}{\sqrt{2(T+\tau)}}}e^{-x^2}dx~.
\end{align*}
We want $\frac{\beta}{\sqrt{2\tau}}\gg 1$ and $\frac{\beta}{\sqrt{2T}}\rightarrow 0$. This can be achieved by setting $\beta=\tau=T^{\frac{\theta}{2}}$ where $\theta\in(0,\frac{2}{3})$. Plugging in $\beta=\tau=T^\frac{\theta}{2}$ we get
\begin{align*}
    \frac{1}{N}\sum_{i=1}^N L(s_{0,i})&\leq
    1-\frac{1}{\sqrt{\pi}}\int^{T^{\theta/4}/\sqrt{2}}_{-\sqrt{\gamma/\left(1 + T^{\theta/2 - 1}\right)}+\frac{T^{\left(\theta-1\right)/2}}{\sqrt{2\left(1 + T^{\theta/2 - 1}\right)}}}e^{-x^2}dx+O(T^{-\frac{\theta}{4}})\\
    &= 1-\frac{1}{\sqrt{\pi}}\int^{\infty}_{-\sqrt{\gamma/\left(1 + T^{\theta/2 - 1}\right)}}e^{-x^2}dx+O(T^{-\frac{\theta}{4}})\\
    &=\frac{1}{\sqrt{\pi}}\int^{\infty}_{\sqrt{\gamma}}e^{-x^2}dx+O(T^{-\frac{\theta}{4}})~.
\end{align*}
\end{proof}

To prove the lower bound we first state the well-known Hoeffding's inequality (c.f. \cite{Hoeffding:1963}).
\begin{lemma}[Hoeffding's inequality]\label{lma: Hoeffding}
Suppose $X_1,\ldots,X_N$ are independent random variables with $X_i$ taking values in $[a_i,b_i]$ for all $i\in[1,N]$. Then for any $\ep>0$, the following inequalities hold for $S_N=\sum_{i=1}^NX_i$
\begin{align*}
    &Pr[S_N-\E[S_N]\geq\ep]\leq e^{-2\ep^2/\sum_{i=1}^N(b_i-a_i)^2}\\
    &Pr[S_N-\E[S_N]\leq-\ep]\leq e^{-2\ep^2/\sum_{i=1}^N(b_i-a_i)^2}
\end{align*}
\end{lemma}

Now we give the proof of the lower bound.

\begin{theorem}\label{thm: lower bound thm 1}
For any $\theta\in(0,\frac{2}{3})$, if the number of chips satisfies
\begin{align*}\label{growth of N, discrete case}
    N>8T^{2+(\theta/2)}\log(1-\sqrt{e^{-T^{-2-(\theta/2)}}})^{-1}=O(T^{2+(\theta/2)}\log T)~,\tag{*}
\end{align*}
then for any player strategy there exists an adversary strategy $\mathcal{A}$ using $\la_L$ with $\beta=\tau=T^{\theta/2}$ such that
\[
    \frac{1}{N}\sum_{i=1}^N L(s_{0,i})\geq\frac{1}{\sqrt{\pi}}\int^{\infty}_{\sqrt{\gamma}}e^{-x^2}dx-O(T^{-\theta/4})~.
\]
\end{theorem}
Now we proceed to lower bound.
\begin{proof}[Proof of Theorem \ref{thm: lower bound thm 1}]
We construct the adversary strategy $\mathcal{A}$ by first introducing a randomized adversary which assigns value to each $z_{i}$ in an \emph{i.i.d} fashion with the following distribution
\begin{align*}
z_{i}=
    \begin{cases}
    1 - \delta\ &w.p.\ \frac{1+\delta+\alpha}{2}\\
    -1 - \delta\ &w.p.\ \frac{1-\delta-\alpha}{2}
    \end{cases}
\end{align*}
where $\alpha>0$ is a parameter to be determined.

For any player strategy $\p$, we have $\E\p\cdot \z = \alpha$ and by Hoeffding's inequality
\[
Pr\bigl[\p\cdot\z < 0\bigr]\leq e^\frac{-\alpha^2}{2\|\p\|^2}\leq e^{-\alpha^2/2}
\]

Next we take the expectation of $\la_L(\s+\z,t+1)-\la_L(\s,t)$,
\begin{align*}
    &\E\left[\la_L(\s+\z,t+1)-\la_L(\s,t)\right]\\
    =&\frac{1}{N}\sum_{i=1}^N \frac{1+\delta}{2}f(s_i+\beta+1-\delta,t+1-\tau)+\frac{1-\delta}{2}f(s_i+\beta-1-\delta,t+1-\tau)-f(s_i+\beta,t-\tau)\\
    &+\frac{\alpha}{2N}\sum_{i=1}^Nf(s_i+\beta+1-\delta, t+1-\tau)-f(s_i+\beta-1-\delta, t+1-\tau)~.
\end{align*}

The second summation can be bounded by simply using the fact that $f\in[0,1]$. To estimate the first summation we use the Taylor expansion
\begin{align*}
    &\frac{1+\delta}{2}f(s_i+\beta+1-\delta,t+1-\tau)+\frac{1-\delta}{2}f(s_i+\beta-1-\delta,t+1-\tau)-f(s_i+\beta,t-\tau)\\
    =&\frac{1+\delta}{2}\left(f(s_i+\beta+1-\delta,t+1-\tau)-f(s_i+\beta,t+1-\tau)\right)\\
    &+\frac{1-\delta}{2}\left(f(s_i+\beta-1-\delta,t+1-\tau)-f(s_i+\beta,t+1-\tau)\right)\\
    &+f(s_i+\beta, t+1-\tau) - f(s_i+\beta,t-\tau)\\
    =& \frac{1-\delta^2}{2}f'(s_i+\beta, t+1-\tau) + \frac{(1-\delta^2)(1-\delta)}{4}f''(s_i+\beta, t+1-\tau)\\
    &-\frac{1-\delta^2}{2}f'(s_i+\beta, t+1-\tau) + \frac{(1-\delta^2)(1+\delta)}{4}f''(s_i+\beta, t+1-\tau) \\
    &+\partial_tf(s_i+\beta, t+1-\tau) + E_i\\
    =& \partial_tf(s_i+\beta, t+1-\tau) + \frac{(1-\delta^2)}{2}f''(s_i+\beta, t+1-\tau) + E_i\\
    =& -\frac{\delta^2}{2}f''(s_i+\beta, t+1-\tau) + E_i,
\end{align*}
where $E_i$ is the remainder term consisting of $\partial_{tt} f$ and $f^{(3)}$. Using Lemma \ref{lma: Control of higher order derivatives of g} and the definition of $\gamma$ in Equation \eqref{eqn: defn-of-gamma, V1 and V2} we can bound
\begin{align*}
    \begin{cases}
        \frac{\delta^2}{2}|f''(s_i+\beta, t+1-\tau)|\leq\frac{C\gamma}{|t+1-\tau|T}\\
        |E_i|\leq \frac{C}{|t+1-\tau|^\frac{3}{2}}+\frac{C}{|t+1-\tau|^2}~,
    \end{cases}
\end{align*}
for some constant $C$. 

By Hoeffding's inequality we also have
\[
Pr\bigl[\la_L(\s+\z,t+1)-\la_L(\s,t) < \E\la_L(\s+\z,t+1)-\la_L(\s,t) -\alpha/2\bigr]\leq e^{-\alpha^2 N / 8}.
\]
Thus when the number of chips $N>\frac{8}{\alpha^2}\log(\frac{1}{1-e^{-\alpha^2/2}})$ we get
\[
e^{-\alpha^2/2}+e^{-\alpha^2 N / 8}<1.
\]
As a consequence, there exists $\z=\left(z_1,\ldots,z_N\right)\in\{-1-\delta, 1-\delta\}^N$ s.t.
\begin{align*}
    \begin{cases}
    \p\cdot\z \geq 0\\
    \la_L(\s+\z,t+1-\tau)-\la_L(\s,t-\tau)\geq -\frac{C\gamma}{|t+1-\tau|T} -\frac{C}{(t+1-\tau)^2}-\frac{C}{|t+1-\tau|^\frac{3}{2}} - \alpha.
    \end{cases}
\end{align*}

Now we can bound the final loss from below
\begin{align*}
    \frac{1}{N}\sum_{i=1}^N L(s_{0,i})\geq\la_L(\s_0,0)&=\la_L(\bm{0}, -T)+\sum_{t=-T}^{-1}\la_L(\s_{t+1}, t+1)-\la_L(\s_t, t)\\
    &\geq\la_L(\bm{0}, -T)-\sum_{t=-T + 1 -\tau}^{-\tau}(\frac{C\gamma}{|t|T} + \frac{C}{|t|^2}+\frac{C}{|t|^\frac{3}{2}} + \alpha)\\
    &=\la_L(\bm{0}, -T)-\alpha T-O(\frac{\gamma\log\tau}{T} + \frac{1}{\sqrt{\tau}}+\frac{1}{\tau})~.
\end{align*}

Note that
\begin{align*}
    \la_L(\bm{0}, -T)&=f(\beta,-T-\tau)-f(\beta-\delta T, -\tau)\\
    &=\frac{1}{\sqrt{\pi}}\int^{\frac{\beta}{\sqrt{2\tau}}}_{\frac{\delta T + \beta}{\sqrt{2(T+\tau)}}}e^{-x^2}dx\\
    &\geq \frac{1}{\sqrt{\pi}}\int^{\frac{\beta}{\sqrt{2\tau}}}_{\frac{\delta T + \beta}{\sqrt{2T}}}e^{-x^2}dx~.
\end{align*}
We want $\frac{\beta}{\sqrt{2\tau}}\gg 1$ and $\frac{\beta}{\sqrt{2T}}\rightarrow 0$. This can be achieved by setting $\beta=\tau=T^{\frac{\theta}{2}}$ where $\theta\in(0,\frac{2}{3})$. Also we want $\alpha T\rightarrow 0$, so we can, for example, let $\alpha=T^{-1 - (\theta/4)}$. As a consequence, for $N>8T^{2+(\theta/2)}\log(1-\sqrt{e^{-T^{-2-(\theta/2)}}})^{-1}$ we have
\begin{align*}
    \frac{1}{N}\sum_{i=1}^N L(s_{0,i})&
    \geq\frac{1}{\sqrt{\pi}}\int^{T^{\theta/4}/\sqrt{2}}_{\sqrt{\gamma}+T^{\left(\theta-1\right)/2}/\sqrt{2}}e^{-x^2}dx-O(T^{-\theta/4})\\
    &=\frac{1}{\sqrt{\pi}}\int^{\infty}_{\sqrt{\gamma}}e^{-x^2}dx-O(T^{-\theta/4})~.
\end{align*}
\end{proof}

\section{Theorems and Proofs for V2}

We first control the derivatives of $\Tilde{g}$ as $t \rightarrow -\infty$, which is a simple corollary of Lemma \ref{lma: Control of higher order derivatives of g}.
\begin{lemma}\label{lma: Control of higher order derivatives of g tilda}
There exists constant $C$ s.t. $\|\tilde{g}'(\cdot,t)\|_{L^\infty(\mathbb{R}^+)}\leq\frac{C}{\sqrt{|t|}}$, $\|\tilde{g}''(\cdot,t)\|_{L^\infty(\mathbb{R}^+)}\leq\frac{C}{|t|}$, $\| \tilde{g}^{(3)}(\cdot,t)\|_{L^\infty(\mathbb{R}^+)}\leq\frac{C}{|t|^\frac{3}{2}}$, and $\|\partial_{tt}\tilde{g}(\cdot,t)\|_{L^\infty(\mathbb{R}^+)}\leq\frac{C}{|t|^2}$.
\end{lemma}

Equipped with the above lemma, we prove upper bound for game V2 with the proposed player strategy (\ref{stgy: continuous action set player strat}).

\begin{theorem}\label{thm:upper bound thm 2}
For any $\theta\in(0,\frac{2}{3})$, the player strategy specified in (\ref{stgy: continuous action set player strat}) with $\tau=T^{\theta}$ satisfies
\begin{align*}
    \frac{1}{N}\sum_{i=1}^N L(s_{0,i})\leq&
    \frac{2}{\sqrt{\pi}}\int^{\infty}_{\sqrt{\gamma}}e^{-x^2}dx+O(T^{-\frac{\theta}{2}})~,
\end{align*}
for any adversary strategy $\mathcal{A}$.
\end{theorem}

\begin{proof}[Proof of Theorem \ref{thm:upper bound thm 2}]
For the sake of simplicity, at any fixed time step $t$, we omit $t$ in the subscripts of $\s,\p,\z,E$. We write the increment of $\la_U$ from $t$ to $t+1$ as
\begin{align*}
    \la_U(\s+\z,t+1)-\la_U(\s,t) &= \frac{1}{N}\sum_{s_i>-\delta T, s_i+z_i>-\delta T} f(s_i+z_i,t+1-\tau) - f(s_i,t-\tau)\\
    &\qquad + \frac{1}{N}\sum_{s_i>-\delta T, s_i+z_i \leq -\delta T} f(s_i+z_i,t+1-\tau) - f(s_i,t-\tau)\\
    &\qquad + \frac{1}{N}\sum_{s_i\leq -\delta T} f(s_i+z_i,t+1-\tau) - f(s_i,t-\tau)\\
    &\leq \frac{1}{N}\sum_{s_i>-\delta T, s_i+z_i>-\delta T} f(s_i+z_i,t+1-\tau) - f(s_i,t-\tau)\\
    &\qquad+ \frac{1}{N}\sum_{s_i>-\delta T, s_i+z_i \leq -\delta T} f(s_i+z_i,t+1-\tau) - f(s_i,t-\tau)\\
    &=: A_1 + A_2~,
\end{align*}
where the inequality is due to the fact that
\[
\sum_{s_i\leq-\delta T} f(s_i+z_i,t+1-\tau) - f(s_i,t-\tau) \leq \sum_{s_i\leq-\delta T} f(s_i+z_i,t+1-\tau) - 1 \leq 0~.
\]

For $A_1$, note that we can apply Taylor expansion since $f(\cdot, t+1-\tau)$ is smooth between $s_i$ and $s_i+z_i$, and $f(s_i, \cdot)$ is smooth on $(t-\tau, t+1-\tau)$. Therefore
\begin{align*}
    A_1 =& \frac{1}{N}\sum_{s_i>-\delta T, s_i+z_i>-\delta T} f(s_i+z_i,t+1-\tau) - f(s_i,t-\tau) \\
    =& \frac{1}{N}\sum_{s_i>-\delta T, s_i+z_i>-\delta T} f(s_i+z_i,t+1-\tau) - f(s_i,t+1-\tau) + f(s_i,t+1-\tau) - f(s_i,t-\tau)\\
    =& \frac{1}{N}\sum_{s_i>-\delta T, s_i+z_i>-\delta T} f'(s_i ,t+1-\tau)z_i + \frac{1}{2} f''(s_i ,t+1-\tau)z_i^2 + \partial_tf(s_i, t+1-\tau)\\
    &+ \frac{1}{N}\sum_{s_i>-\delta T, s_i+z_i>-\delta T}E_i~,
\end{align*}
where $E_i$ is the remainder term consisting of $\partial_{tt} f$ and $f^{(3)}$.

For $A_2$, first note that when $s_i+z_i\leq -\delta T$,
\[
f(s_i+z_i,t+1-\tau) = 1 = f(-\delta T,t+1-\tau)~.
\]
Therefore we can write
\begin{align*}
    A_2 =& \frac{1}{N}\sum_{s_i>-\delta T, s_i+z_i \leq -\delta T} f(s_i+z_i,t+1-\tau) - f(s_i,t-\tau)\\
    =& \frac{1}{N}\sum_{s_i> -\delta T, s_i+z_i \leq -\delta T} f(-\delta T, t+1-\tau) - f(s_i,t-\tau)~.
\end{align*}
Since $f(\cdot, t+1-\tau)$ is smooth on $(-\delta T, s_i)$ and $f(s_i, \cdot)$ is smooth on $(t-\tau, t+1-\tau)$ we use Taylor expansion on $A_2$
\begin{align*}
    A_2 =& \frac{1}{N}\sum_{s_i> -\delta T, s_i+z_i \leq -\delta T} f(-\delta T, t+1-\tau) - f(s_i,t-\tau)\\
    =& \frac{1}{N}\sum_{s_i> -\delta T, s_i+z_i \leq -\delta T} f(-\delta T,t+1-\tau) - f(s_i,t+1-\tau) + f(s_i,t+1-\tau) - f(s_i,t-\tau)\\
    =& \frac{1}{N}\sum_{s_i>-\delta T, s_i+z_i \leq -\delta T} -f'(s_i ,t+1-\tau)(\delta T + s_i) + \frac{1}{2} f''(s_i ,t+1-\tau)(\delta T+s_i)^2 + \partial_tf(s_i, t+1-\tau)\\
    &+ \frac{1}{N}\sum_{s_i>-\delta T, s_i+z_i\leq -\delta T}E_i~,
\end{align*}
where $E_i$ is the remainder term consisting of $\partial_{tt} f$ and $f^{(3)}$. Note that $s_i> -\delta T, s_i+z_i \leq -\delta T$ implies $-z_i \geq s_i + \delta T > 0$.  Also since $f'(s_i, t+1-\tau)\leq 0$ and $f''(s_i, t+1-\tau) \geq 0$ when $s_i > -\delta T$, we have the following inequalities
\begin{align*}
    \begin{cases}
        -(\delta T + s_i)f'(s_i, t+1-\tau) \leq z_i f'(s_i, t+1-\tau)\\
        f''(s_i ,t+1-\tau)(\delta T+s_i)^2 \leq f''(s_i ,t+1-\tau)z_i^2~.
    \end{cases}
\end{align*}
As a consequence
\begin{align*}
A_2 \leq& \frac{1}{N}\sum_{s_i>-\delta T, s_i+z_i \leq -\delta T} f'(s_i ,t+1-\tau)z_i + \frac{1}{2} f''(s_i ,t+1-\tau)z_i^2 + \partial_tf(s_i, t+1-\tau)\\
&+ \frac{1}{N}\sum_{s_i>-\delta T, s_i+z_i\leq -\delta T}E_i~,
\end{align*}

Combining the inequalities for $A_1,A_2$, we get
\begin{align*}
    \la_U(\s+\z,t+1)-\la_U(\s,t) \leq& \frac{1}{N}\sum_{s_i>-\delta T} f'(s_i ,t+1-\tau)z_i + \frac{1}{2} f''(s_i ,t+1-\tau)z_i^2 + \partial_tf(s_i, t+1-\tau)\\
    &+ \frac{1}{N}\sum_{s_i>-\delta T}E_i~.\\
    =:& A_3+\frac{1}{N}\sum_{s_i>-\delta T}E_i~.
\end{align*}
For the remainder $\frac{1}{N}\sum_{s_i>-\delta T}E_i$, using Lemma \ref{lma: Control of higher order derivatives of g tilda} and the fact that $|z_i| \leq 1+\delta,i\in[N]$ we can conclude there exists a constant $C$ such that
\[
|\frac{1}{N}\sum_{s_i>-\delta T}E_i| \leq \sup_{s_i > -\delta T}|E_i| \leq \frac{C}{|t+1-\tau|^\frac{3}{2}}+\frac{C}{|t+1-\tau|^2}~.
\]

Recalling that the player set $p_i=-\frac{f'(s_i ,t+1-\tau)}{\sum_{s_j > -\delta T}f'(s_j, t+1-\tau)}$ for $s_i > -\delta T$ and that $\p\cdot\z \geq 0$, we have
\[
\sum_{s_i>-\delta T} f'(s_i ,t+1-\tau)z_i \leq 0~.
\]
Therefore
\begin{align*}
   A_3\leq& \frac{1}{N}\sum_{s_i>-\delta T} \partial_tf(s_i, t+1-\tau) + \frac{1}{2} f''(s_i ,t+1-\tau) + \frac{z_i^2-1}{2} f''(s_i ,t+1-\tau)\\
   \leq& \frac{1}{N}\sum_{s_i>-\delta T}\frac{z_i^2-1}{2} f''(s_i ,t+1-\tau)\\
   \leq& \frac{1}{N}\sum_{s_i>-\delta T}\frac{3\delta}{2} f''(s_i ,t+1-\tau)\\
   \leq& \frac{C\sqrt{\gamma}}{|t+1-\tau|\sqrt{T}}~.
\end{align*}
The second inequality used the fact that $f$ satisfies \eqref{eqn: backward nonlinear equation} when $s_i > -\delta T$; the third inequality used the fact that $f''(s_i, t+1-\tau) \geq 0$ and $z_i^2 -1 \leq 2\delta + \delta^2 \leq 3\delta$; the last inequality used Equation \eqref{eqn: defn-of-gamma, V1 and V2} and Lemma \ref{lma: Control of higher order derivatives of g tilda}.

Combining the above analysis together we now add up the increment of $\la_U$ from $-T$ to $-1$,
\begin{align*}
    \frac{1}{N}\sum_{i=1}^N L(s_{0,i})\leq\la_U(\s_0,0)&=\la_U(\bm{0}, -T)+\sum_{t=-T}^{-1}\la_U(\s_{t+1}, t+1)-\la_U(\s_t, t)\\
    &\leq\la_U(\bm{0}, -T)+C\sum_{t=-T+1-\tau}^{-\tau}\frac{\sqrt{\gamma}}{|t|\sqrt{T}} + \frac{1}{|t|^\frac{3}{2}}+\frac{1}{|t|^2}\\
    &\leq\la_U(\bm{0}, -T)+O\left(\frac{\log\tau\sqrt{\gamma}}{\sqrt{T}} + \frac{1}{\sqrt{\tau}}+\frac{1}{\tau}\right)~.
\end{align*}
The main temr on the right is 
\begin{align*}
    \la_U(\bm{0}, -T)=f(0,-T-\tau)=\frac{2}{\sqrt{\pi}}\int_{\delta T/\sqrt{2(T+\tau)}}^\infty e^{-x^2} dx
\end{align*}
We set $\tau=T^\theta$ where $\theta\in (0, \frac{2}{3})$, then
\begin{align*}
    \frac{1}{N}\sum_{i=1}^N L(s_{0,i})\leq&
    \frac{2}{\sqrt{\pi}}\int^{\infty}_{\sqrt{\frac{\gamma}{1+T^{\theta-1}}}}e^{-x^2}dx+O(T^{-\frac{\theta}{2}})\\
    =&\frac{2}{\sqrt{\pi}}\int^{\infty}_{\sqrt{\gamma}}e^{-x^2}dx+O(T^{-\frac{\theta}{2}})~.
\end{align*}
\end{proof}


Before proving lower bound we give another key ingredient.
\begin{lemma}\label{lma: combinatoric lemma}[Lemma 2 of \cite{ContExperts}]
For any sequence $a_1,\ldots,a_n$ belonging to $[0, U]$ for some constant $U>0$, the following holds
\begin{align*}
    \min_{P\subset [n]} |\sum_{i\in P}a_i - \sum_{j\in [n]\backslash P}a_j| \leq U~.
\end{align*}
\end{lemma}

Now we give the proof of the lower bound
\begin{theorem}\label{thm: lower bound thm 2}
For any $T$ and $\theta\in(0, \frac{2}{3})$, if the number of chips $N \geq T^\frac{\theta+2}{4}$, then for any player strategy there exists an adversary strategy $\mathcal{A}$ associated with $\la_L$ with
\[
\beta=\tau=\lceil \frac{\delta T + T^{\frac{\theta}{2}}}{1-\delta} \rceil (1-\delta) -\delta T
\]
such that
\begin{equation*}
\frac{1}{N}\sum_{i=1}^N L(s_{0,i})\geq\frac{2}{\sqrt{\pi}}\int^{\infty}_{\sqrt{\gamma}}e^{-x^2}dx-O(T^{-\frac{\theta}{4}})~.
\end{equation*}
Moreover, $\mathcal{A}$ takes $z_i=0$ when $s_i \leq - \lceil \frac{\delta T + T^{\frac{\theta}{2}}}{1-\delta} \rceil (1-\delta)$ and $|z_i|=1-\delta$ otherwise. 
\end{theorem}
\begin{proof}[Proof of Theorem \ref{thm: lower bound thm 2}]
We consider an adversary that only takes $\{-1+\delta,0,1-\delta\}$.  With this choice of action set the chips always lie on multiples of $1-\delta$. Moreover, our adversary assigns $z_i=0$ whenever $s_i \leq -\delta T - \beta$ and $|z_i|=1-\delta$ otherwise.

We bound the increment of $\la_L$ as follows. First note that when $s_i \leq - \delta T - \beta$, the adversary chooses $z_i=0$, which implies that
\[
f(s_i+z_i+\beta,t+1-\tau)=f(s_i+\beta,t-\tau)=1~.
\]

As a consequence
\[
\la_L(\s+\z,t+1)-\la_L(\s,t)
    =\frac{1}{N}\sum_{s_i > - \delta T - \beta} f(s_i+z_i+\beta , t+1-\tau) - f(s_i +\beta, t-\tau)~.
\]

Also note that $s_i$ is a multiple of $1-\delta$ and by our choice of $\beta$, $- \delta T - \beta=-\lceil \frac{\delta T + T^{\frac{\theta}{2}}}{1-\delta}\rceil(1-\delta)$ is also a multiple of $1-\delta$. As a consequence when $s_i > - \delta T - \beta$ we have $s_i + z_i \geq s_i - (1-\delta) \geq - \delta T - \beta$. Therefore $f(\cdot, t+1-\tau)$ is smooth between $s_i+\beta$ and $s_i+z_i+\beta$, and $f(s_i+\beta,\cdot)$ is smooth on $(t-\tau, t+1-\tau)$. We apply Taylor expansion in the case  $s_i > - \delta T - \beta$
\begin{align*}
    &\frac{1}{N}\sum_{s_i > - \delta T - \beta} f(s_i + z_i + \beta , t+1-\tau) - f(s_i + \beta, t-\tau)\\
    =&\frac{1}{N}\sum_{s_i > - \delta T - \beta}\left(z_if'(s_i + \beta, t+1-\tau) + \partial_tf(s_i + \beta, t+1-\tau)+\frac{(1-\delta)^2}{2}f''(s_i + \beta, t+1-\tau)\right)\\
    &+\frac{1}{N}\sum_{s_i > - \delta T - \beta}E_i\\
    =&\frac{1}{N}\sum_{s_i > - \delta T - \beta}z_if'(s_i + \beta, t+1-\tau) + \frac{-2\delta+\delta^2}{2}f''(s_i + \beta, t+1-\tau) + \frac{1}{N}\sum_{s_i > - \delta T - \beta}E_i\\
    \geq& \frac{1}{N}\sum_{s_i > - \delta T - \beta}z_if'(s_i + \beta, t+1-\tau) - \delta f''(s_i + \beta, t+1-\tau) + \frac{1}{N}\sum_{s_i > - \delta T - \beta}E_i~,
\end{align*}
In the second equality we used the fact that $f$ satisfies Equation \eqref{eqn: backward nonlinear equation} and in the inequality we used the fact that $f''(s_i + \beta, t+1-\tau) \geq 0$.

$E_i$ is the remainder term consisting of $\partial_{tt} f$ and $f^{(3)}$. Using Lemma \ref{lma: Control of higher order derivatives of g tilda} and the fact that $|z_i| \leq 1-\delta$ there exists a constant $C$ such that
\[
|\frac{1}{N}\sum_{s_i > - \delta T - \beta}E_i| \leq \sup_{s_i > - \delta T - \beta}|E_i| \leq \frac{C}{|t+1-\tau|^\frac{3}{2}}+\frac{C}{|t+1-\tau|^2}~.
\]
We can bound the second order term
\[
- \delta f''(s_i + \beta, t+1-\tau) \geq -\frac{C\sqrt{\gamma}}{|t+1-\tau|\sqrt{T}}~,
\]
by Equation \eqref{eqn: defn-of-gamma, V1 and V2} and Lemma \ref{lma: Control of higher order derivatives of g tilda}.

To bound the first order term $\sum_{s_i > - \delta T - \beta}z_if'(s_i+\beta, t+1-\tau)$ we apply Lemma \ref{lma: combinatoric lemma}. More specifically, in our case $a_i=-f'(s_i+\beta, t+1-\tau)$ and $U=\frac{C}{\sqrt{|t+1-\tau|}}$ by Lemma \ref{lma: Control of higher order derivatives of g tilda}. Lemma \ref{lma: combinatoric lemma} confirms there exists a subset $P\subset\{i:s_i > - \delta T - \beta\}$ (note that we can always make the first inequality below holds by swapping $P$ and $\{k:s_k > - \delta T - \beta\}\backslash P$) such that
\begin{align*}
    \begin{cases}
        \sum_{i\in P}p_i - \sum_{j\in \{k:s_k > - \delta T - \beta\}\backslash P} p_j \geq 0\\
        |\sum_{i\in P}a_i - \sum_{j\in \{k:s_k > - \delta T - \beta\}\backslash P}a_j| \leq U
    \end{cases}
\end{align*}
Thus by setting $z_i=1-\delta$ for $i\in P$ and $z_i=-1+\delta$ for $i\in\{k:s_k > - \delta T - \beta\}\backslash P$, the adversary makes
\begin{align*}
    \begin{cases}
    \sum_{s_i > - \delta T - \beta}z_i f'(s_i+\beta, t+1-\tau) \geq - \frac{C}{\sqrt{|t+1-\tau|}}\\
    \sum_{s_i > - \delta T - \beta}p_i\cdot z_i \geq 0
    \end{cases}
\end{align*}
Moreover since $z_i=0$ for $s_i \leq - \delta T - \beta$, we have
\begin{align*}
    \begin{cases}
    \sum_{i}z_i f'(s_i+\beta, t+1-\tau) \geq - \frac{C}{\sqrt{|t+1-\tau|}}\\
    \p\cdot\z \geq 0
    \end{cases}
\end{align*}

As a consequence, we can bound the final loss from below
\begin{align*}
    \frac{1}{N}\sum_{i=1}^N L(s_{0,i}) \geq \la_L(\s_0,0)&=\la_L(\bm{0}, -T)+\sum_{t=-T}^{-1}\la_L(\s_{t+1}, t+1)-\la_L(\s_t, t)\\
     &\geq \la_L(\bm{0}, -T)-C\sum_{t=-T+1-\tau}^{-\tau}\frac{1}{N|t|^{\frac{1}{2}}}+\frac{\sqrt{\gamma}}{|t|\sqrt{T}}+\frac{1}{|t|^\frac{3}{2}}+\frac{1}{|t|^2}\\
    &= \la_L(\bm{0}, -T) - O\left(\frac{\sqrt{T+\tau}}{N}+\frac{\sqrt{\gamma}\log\tau}{\sqrt{T}}+\frac{1}{\sqrt{\tau}}+\frac{1}{\tau}\right)~.
\end{align*}

Now we compute the main term
\begin{align*}
    \la_L(\bm{0}, -T)=&f(\beta,-T-\tau) - f(-\delta T+\beta, -\tau)\\
    =&\frac{2}{\sqrt{\pi}}\int^{\frac{\beta}{\sqrt{2\tau}}}_{\frac{\delta T+\beta}{\sqrt{2(T+\tau)}}}e^{-x^2}dx\\
    \geq&\frac{2}{\sqrt{\pi}}\int^{\frac{\beta}{\sqrt{2\tau}}}_{\frac{\delta T+\beta}{\sqrt{2T}}}e^{-x^2}dx~.
\end{align*}

Finally plugging in $\beta=\tau=\lceil \frac{\delta T + T^{\frac{\theta}{2}}}{1-\delta} \rceil (1-\delta) -\delta T$ and $N=T^\frac{2+\theta}{4}$ for $\theta \in (0, \frac{2}{3})$, we get
\begin{align*}
    \frac{1}{N}\sum_{i=1}^N L(s_{0,i})
    &\geq\frac{2}{\sqrt{\pi}}\int^{T^{\theta/4}/\sqrt{2}}_{\sqrt{\gamma}+\left(T^\frac{\theta}{2}+1\right)/\sqrt{2T}}e^{-x^2}dx-O(T^{-\frac{\theta}{4}})\\
    &=\frac{2}{\sqrt{\pi}}\int^{\infty}_{\sqrt{\gamma}}e^{-x^2}dx-O(T^{-\frac{\theta}{4}})~.
\end{align*}
\end{proof}

\section{Game V3}\label{Sec: discrete action set 2}
In this section we consider game V3: $\delta = 0, \mathcal{Z} = \{\pm 1\}$, and $L(s)=\ind{s \leq -R}$ for some constant $R>0$. Game V3 is related to the {\em prediction with expert advice} game in which the experts make binary decisions. We give new player and adversary strategies for this game (thus also for {\em prediction with expert advice) and give (matching) upper and lower bounds which can be seen as the limit of the discrete bounds given in \cite{Cesa-Bianchi96} as $T\rightarrow\infty$. 

The technical details here are very similar to Section \ref{Sec: discrete action set}. We use 
\[
f(s,t)= g(s + R,t) = \frac{1}{\sqrt{\pi}}\int_{\frac{s + R}{\sqrt{-2t}}}^\infty e^{-x^2} ~,
\]
which satisfies the following PDE, 
\begin{align}\label{eqn: equation for discrete action set 2}
\begin{cases}
\partial_t f(s, t)+\frac{1}{2}f''(s, t) = 0\\
f(s, 0) = \ind{\s \leq - R}~,
\end{cases}
\end{align}
and serves as the building block.

In the statements of the theorems in this section $\gamma$ is defined as in Equation \eqref{eqn: defn-of-gamma, V3 and V4}, i.e.
\[
\gamma = \frac{R^2}{2T}~,
\]
and it is a constant as $T\rightarrow\infty$.

\subsection{Potentials}
We use a shifted version of $f$ for the upper bound potential $\la_U$.
\begin{align*}
    \la_U(\s, t) = \frac{1}{N}\sum_{i=1}^N f(s_i - \beta, t - \tau) + 1 - f(-R - \beta, -\tau)~,
\end{align*}
with $\beta=\tau=T^\frac{\theta}{2}$ for any $\theta\in(0,\frac{2}{3})$. Note that $\la_U$ is a decreasing function on all the spatial coordinates. This implies that $\p_t\sim-\nabla\la_U(\s, t+1)$ is a valid probability distribution.

We use a different shift of $f$ to construct the lower bound potential. More specifically, we define
\[
\la_L(\s, t) = \frac{1}{N} \sum_{i=1}^N f(s_i + \beta, t - \tau) - f(-R + \beta, -\tau)~.
\]
with $\beta=\tau=T^{\theta/2}$ for any $\theta\in(0,\frac{2}{3})$. Compared to the other binary case V1, the adversary's choices here ($\{\pm 1\}$) are symmetric about 0, so we can derive lower bound using Lemma \ref{lma: combinatoric lemma}, given $N \geq T^\frac{2+\theta}{4}$.

\subsection{Theorems and Proofs}
We first give the proof of upper bound 

\begin{theorem}\label{thm:upper bound thm 3'}
For any $\theta\in(0,\frac{2}{3})$ the player strategy proportional to the negative gradient of $\la_U$ with $\beta=\tau=T^\frac{\theta}{2}$ satisfies
\[
\frac{1}{N}\sum_{i=1}^N L(s_{0,i})\leq
    \frac{1}{\sqrt{\pi}}\int^{\infty}_{\sqrt{\gamma}}e^{-x^2}dx + O(T^{-\frac{\theta}{4}})~,
\]
for any adversary strategy $\mathcal{A}$. 
\end{theorem}

\begin{proof}[Proof of Theorem \ref{thm:upper bound thm 3'}]
We write the increment of $\la_U$ from $t$ to $t+1$ as
\begin{align*}
    &\la_U(\s+\z, t+1)-\la_U(\s, t)\\
    =& \frac{1}{N}\sum_{i=1}^N f(s_i-\beta+z_i,t+1-\tau) - f(s_i-\beta,t-\tau)\\
    =& \frac{1}{N}\sum_{i=1}^N f(s_i-\beta+z_i,t+1-\tau) - f(s_i-\beta,t+1-\tau) + f(s_i-\beta,t+1-\tau) - f(s_i-\beta,t-\tau)\\
    =& \frac{1}{N}\sum_{i=1}^N f'(s_i-\beta ,t+1-\tau)z_i + \frac{1}{2} f''(s_i-\beta ,t+1-\tau)z_i^2 + \partial_tf(s_i-\beta, t+1-\tau)+ \frac{1}{N}\sum_{i=1}^N E_i~.
\end{align*}
The remainder term $E_i$ consists of $f^{(3)}$ and $\partial_{tt}f$.

Recalling that the player set $p_i=-\frac{f'(s_i-\beta ,t+1-\tau)}{\sum_{j=1}^N f'(s_j-\beta, t+1-\tau)}$ and that $\p\cdot\z \geq 0$, we have
\[
\sum_{i=1}^N f'(s_i-\beta ,t+1-\tau)z_i \leq 0~.
\]
Therefore
\begin{align*}
   &\frac{1}{N}\sum_{i=1}^N f'(s_i-\beta ,t+1-\tau)z_i + \frac{1}{2} f''(s_i-\beta ,t+1-\tau)z_i^2 + \partial_tf(s_i-\beta, t+1-\tau)\\
   \leq& \frac{1}{N}\sum_{i=1}^N \partial_tf(s_i-\beta, t+1-\tau) + \frac{1}{2} f''(s_i-\beta ,t+1-\tau)\\
   \leq& 0~,
\end{align*}
where the first inequality used the fact that $|z_i|^2=1$ and the second inequality used Equation \eqref{eqn: equation for discrete action set 2}.

For the remainder $\frac{1}{N}\sum_{i=1}^N E_i$, using Lemma \ref{lma: Control of higher order derivatives of g} and the fact that $|z_i| \leq 1,i\in[N]$ we can conclude there exists a constant $C$ such that
\[
|\frac{1}{N}\sum_{i=1}^N E_i| \leq \sup_{i\in[N]}|E_i| \leq \frac{C}{|t+1-\tau|^\frac{3}{2}}+\frac{C}{|t+1-\tau|^2}~.
\]

Combining the above analysis together we now add up the increment of $\la_U$ from $-T$ to $-1$,
we have
\begin{align*}
    \frac{1}{N}\sum_{i=1}^N L(s_{0,i})\leq\la_U(\s_0,0)&=\la_U(\bm{0}, -T)+\sum_{t=-T}^{-1}\la_U(\s_{t+1}, t+1)-\la_U(\s_t, t)\\
    &\leq\la_U(\bm{0}, -T)+C\sum_{t=-T+1-\tau}^{-\tau}\frac{1}{|t|^2}+\frac{1}{|t|^\frac{3}{2}}\\
    &\leq\la_U(\bm{0}, -T)+O\left(\frac{1}{\sqrt{\tau}}+\frac{1}{\tau}\right)~.
\end{align*}
Note that
\begin{align*}
    \la_U(0, -T)&=f(-\beta,-T-\tau)+1-f(-R-\beta, -\tau)\\
    &=1-\frac{1}{\sqrt{\pi}}\int^{\frac{\beta}{\sqrt{2\tau}}}_{\frac{\beta-R}{\sqrt{2(T+\tau)}}}e^{-x^2}dx~.
\end{align*}
Plugging in $\beta=\tau=T^\frac{\theta}{2}$ for $\theta \in (0, \frac{2}{3})$ we get
\begin{align*}
    \frac{1}{N}\sum_{i=1}^N L(s_{0,i})
    &\leq
    1-\frac{1}{\sqrt{\pi}}\int^{T^{\theta/4}/\sqrt{2}}_{-\sqrt{\gamma/\left(1 + T^{\theta/2 - 1}\right)}+\frac{T^{\left(\theta-1\right)/2}}{\sqrt{2\left(1 + T^{\theta/2 - 1}\right)}}}e^{-x^2}dx+O(T^{-\frac{\theta}{4}})\\
    &=1-\frac{1}{\sqrt{\pi}}\int^{\infty}_{-\sqrt{\gamma/\left(1 + T^{\theta/2 - 1}\right)}}e^{-x^2}dx+O(T^{-\frac{\theta}{4}})\\
    &=1-\frac{1}{\sqrt{\pi}}\int^{\infty}_{-\sqrt{\gamma}}e^{-x^2}dx+O(T^{-\frac{\theta}{4}})\\
    &=\frac{1}{\sqrt{\pi}}\int^{\infty}_{\sqrt{\gamma}}e^{-x^2}dx+O(T^{-\frac{\theta}{4}})~.    
\end{align*}
\end{proof}

Next we give the proof of lower bound
\begin{theorem}\label{thm: lower bound thm 3'}
For any $T$ and any $\theta\in(0,\frac{2}{3})$, if the number of chips satisfies
\[
N \geq T^\frac{2+\theta}{4}~,
\]
then for any player strategy there exists an adversary strategy $\mathcal{A}$ associated with $\la_L$ (using $\beta=\tau=T^{\theta/2}$) such that
\[
    \frac{1}{N}\sum_{i=1}^N L(s_{0,i})\geq\frac{1}{\sqrt{\pi}}\int^{\infty}_{\sqrt{\gamma}}e^{-x^2}dx-O(T^{-\theta/4})~.
\]
\end{theorem}

\begin{proof}[Proof of Theorem \ref{thm: lower bound thm 3'}]
We bound the increment of $\la_L$ using Taylor expansion 
\begin{align*}
    &\la_L(\s+\z,t+1)-\la_L(\s,t)\\
    =&\frac{1}{N}\sum_{i=1}^N f(s_i+z_i+\beta , t+1-\tau) - f(s_i +\beta, t-\tau)\\
    =&\frac{1}{N}\sum_{i=1}^N\left(z_if'(s_i + \beta, t+1-\tau) + \partial_tf(s_i + \beta, t+1-\tau)+\frac{1}{2}f''(s_i + \beta, t+1-\tau)\right)\\
    &+\frac{1}{N}\sum_{i=1}^NE_i\\
    =&\frac{1}{N}\sum_{i=1}^Nz_if'(s_i + \beta, t+1-\tau) + \frac{1}{N}\sum_{i=1}^NE_i~,
\end{align*}
In the second equality we used the fact that $|z_i|=1$ and in the last equality we used the fact that $f$ satisfies Equation \eqref{eqn: equation for discrete action set 2}.
The remainder term $E_i$ consists of $f^{(3)}(\cdot, t+1-\tau)$ and $\partial_{tt}f(s_i-\beta, \cdot)$. Using Lemma \ref{lma: Control of higher order derivatives of g} and the fact that $|z_i| = 1$ there exists a constant $C$ such that
\[
|\frac{1}{N}\sum_{i=1}^N E_i| \leq \sup_{i\in[N]}|E_i| \leq \frac{C}{|t+1-\tau|^\frac{3}{2}}+\frac{C}{|t+1-\tau|^2}~.
\]

To bound the first order term $\sum_{i=1}^Nz_if'(s_i+\beta, t+1-\tau)$ we apply Lemma \ref{lma: combinatoric lemma}. More specifically, in our case $a_i=-f'(s_i+\beta, t+1-\tau)$ and $U=\frac{C}{\sqrt{|t+1-\tau|}}$ by Lemma \ref{lma: Control of higher order derivatives of g}. Lemma \ref{lma: combinatoric lemma} confirms there exists a subset $P\subset[N]$ (note that we can always make the first inequality below holds by swapping $P$ and $[N]\backslash P$) such that 
\begin{align*}
    \begin{cases}
        \sum_{i\in P}p_i - \sum_{j\in [N]\backslash P} p_j \geq 0\\
        |\sum_{i\in P}a_i - \sum_{j\in [N]\backslash P}a_j| \leq U
    \end{cases}
\end{align*}
Thus by setting $z_i=1$ for $i\in P$ and $z_i=-1$ for $i\in[N]\backslash P$, the adversary can arrange that
\begin{align*}
    \begin{cases}
    \sum_{i=1}^Nz_i f'(s_i+\beta, t+1-\tau) \geq - \frac{C}{\sqrt{|t+1-\tau|}}\\
    \sum_{i=1}^Np_i\cdot z_i \geq 0~.
    \end{cases}
\end{align*}

\noindent
As a consequence, we bound the final loss from below:
\begin{align*}
    \frac{1}{N}\sum_{i=1}^N L(s_{0,i})\geq\la_L(\s_0,0)&=\la_L(\bm{0}, -T)+\sum_{t=-T}^{-1}\la_L(\s_{t+1},t+1)-\la_L(\s_t,t)\\
    &\geq \la_L(\bm{0}, -T)-C\sum_{t=-T+1-\tau}^{-\tau}\frac{1}{N|t|^{\frac{1}{2}}}+\frac{1}{|t|^\frac{3}{2}}+\frac{1}{|t|^2}\\
    &=\la_L(\bm{0}, -T) - O\left(\frac{\sqrt{T+\tau}}{N} + \frac{1}{\sqrt{\tau}}+\frac{1}{\tau}\right)~.
\end{align*}

\noindent
Now we compute the main term
\begin{align*}
    \la_L(\bm{0}, -T)=&f(\beta,-T-\tau) - f(-R+\beta, -\tau)\\
    =&g(\beta + R, -T-\tau) - g(\beta, -\tau)\\
    =&\frac{1}{\sqrt{\pi}}\int^{\frac{\beta}{\sqrt{2\tau}}}_{\frac{R+\beta}{\sqrt{2(T+\tau)}}}e^{-x^2}dx\\
    \geq& \frac{1}{\sqrt{\pi}}\int^{\frac{\beta}{\sqrt{2\tau}}}_{\frac{R+\beta}{\sqrt{2T}}}e^{-x^2}dx~.
\end{align*}

\noindent
Finally plugging in $\beta=\tau=T^{\frac{\theta}{2}}$ and $N=T^\frac{2+\theta}{4}$ for $\theta\in(0, \frac{2}{3})$ we get,
\begin{align*}
    \frac{1}{N}\sum_{i=1}^N L(s_{0,i})
    &\geq\frac{1}{\sqrt{\pi}}\int^{T^{\frac{\theta}{4}}/\sqrt{2}}_{\sqrt{\gamma}+T^\frac{\theta-1}{2}/\sqrt{2}}e^{-x^2}dx-O(T^{-\frac{\theta}{4}})\\
    &=\frac{1}{\sqrt{\pi}}\int^{\infty}_{\sqrt{\gamma}}e^{-x^2}dx-O(T^{-\frac{\theta}{4}})~.
\end{align*}
\end{proof}

\section{Game V4}\label{Appendix: continuous action set variant 2}
In this section we consider game V4: $\delta = 0, \mathcal{Z} = [-1,1]$, and $L(s)=\ind{s \leq -R}$ for some constant $R>0$. This game is related to a \emph{prediction with expert advice} setting in which each expert's prediction takes a value on the interval $[-1,1]$; it was considered in \cite{ContExperts} and is also related to the \emph{hedge game} \cite{HedgeGame}. We give new player and adversary strategies for this game (thus also for the two related games) and give (matching) upper and lower bounds which can be seen as the limit of the discrete bounds given in \cite{ContExperts} as $T\rightarrow\infty$.


Motivated by the upper bound potential defined in Section \ref{subsec:potentials of V2}, we define
\begin{align*}
f(s,t) = \tilde{g}(s + R, t)~,
\end{align*}
which solves
\begin{align}\label{eqn: equation for continuous action set 2}
\begin{cases}
\partial_t f(s, t)+\frac{1}{2}\max(f''(s, t),0) = 0\\
f(s, 0) =  \ind{\s \leq - R}~.
\end{cases}
\end{align}
and serves as the building block.

As in Appendix \ref{Sec: discrete action set 2}, in this Appendix
\[
\gamma := \frac{R^2}{2T}~,
\]
and $\gamma$ is held constant as $T\rightarrow \infty$.

\subsection{Potentials}
We define the upper bound potential $\la_U$ using a shifted version of $f$
\[
\la_U(\s,t) = \frac{1}{N}\sum_{i=1}^N f(s_i, t-\tau)~,
\]
where $\tau=T^\theta$ for any $\theta \in (0,\frac{2}{3})$. $\la_U$ is a decreasing function in all the spatial variables.

The player imposes the following distribution: if at least one chip is on the right of $-R$,
\begin{align}\label{stgy: player strat}
    p_{i} \sim
    \begin{cases}
        0&s_i \leq -R\\
        -f'(s_i, t+1-\tau)&s_i > -R~;
    \end{cases}
\end{align}
otherwise if all the chips are on the left of $-R$, any probability distribution is fine.

We define the lower bound potential $\la_L$ as
\[
\la_L(\s, t) = \frac{1}{N} \sum_{i=1}^N f(s_i + \beta, t - \tau) - f(-R + \beta, -\tau)~,
\]
with $\beta=\tau=\lceil R + T^{\theta/2} \rceil - R$ for any $\theta \in (0, \frac{2}{3})$. We will again use Lemma \ref{lma: combinatoric lemma} to derive a lower bound, provided $N \geq T^\frac{2+\theta}{4}$.

\subsection{Theorems and Proofs}
We first give the proof of upper bound
\begin{theorem}\label{thm:upper bound thm 3}
For any $\theta\in(0, \frac{2}{3})$, the player strategy following (\ref{stgy: player strat}) with $\tau=T^\theta$ satisfies
\begin{align*}
    \frac{1}{N}\sum_{i=1}^N L(s_{0,i})\leq&
    \frac{2}{\sqrt{\pi}}\int^{\infty}_{\sqrt{\gamma}}e^{-x^2}dx+O(T^{-\frac{\theta}{2}})~,
\end{align*}
for any adversary strategy $\mathcal{A}$. 
\end{theorem}

\begin{proof}[Proof of Theorem \ref{thm:upper bound thm 3}]
We follow the proof in Theorem \ref{thm:upper bound thm 2} and replace $\delta T$ by $R$. The increment of the potential from $t$ to $t+1$ is bounded as
\begin{align*}
    \la_U(\s+\z,t+1)-\la_U(\s,t) \leq& \frac{1}{N}\sum_{s_i>-R, s_i+z_i>-R} f(s_i+z_i,t+1-\tau) - f(s_i,t-\tau)\\
    &+ \frac{1}{N}\sum_{s_i>-R, s_i+z_i \leq -R} f(s_i+z_i,t+1-\tau) - f(s_i,t-\tau)\\
    =&: A_1 + A_2~.
\end{align*}

For $A_1$, we apply Taylor expansion and get
\begin{align*}
    A_1 =& \frac{1}{N}\sum_{s_i>-R, s_i+z_i>-R} f'(s_i ,t+1-\tau)z_i + \frac{1}{2} f''(s_i ,t+1-\tau)z_i^2 + \partial_tf(s_i, t+1-\tau)\\
    &+ \frac{1}{N}\sum_{s_i>-R, s_i+z_i>-R}E_i~,
\end{align*}
where $E_i$, in the case of $s_i>-R, s_i+z_i>-R$, is the remainder term consisting of $f^{(3)}$ and $\partial_{tt}f$.

For $A_2$, we can write
\begin{align*}
    A_2 =& \frac{1}{N}\sum_{s_i>-R, s_i+z_i \leq -R} f'(s_i ,t+1-\tau)(-R - s_i) + \frac{1}{2} f''(s_i ,t+1-\tau)(-R-s_i)^2 + \partial_tf(s_i, t+1-\tau)\\
    &+ \frac{1}{N}\sum_{s_i>-R, s_i+z_i \leq -R}E_i\\
    \leq &\frac{1}{N}\sum_{s_i>-R, s_i+z_i \leq -R} f'(s_i ,t+1-\tau)z_i + \frac{1}{2} f''(s_i ,t+1-\tau)z_i^2 + \partial_tf(s_i, t+1-\tau)\\
    &+ \frac{1}{N}\sum_{s_i>-R, s_i+z_i \leq -R}E_i~.
\end{align*}
The remainder term $E_i$ consists of $f^{(3)}$ and $\partial_{tt}f$, and the inequality holds since
\begin{align*}
    \begin{cases}
        -(R + s_i)f'(s_i, t+1-\tau) \leq z_i f'(s_i, t+1-\tau)\\
        f''(s_i ,t+1-\tau)(R +s_i)^2 \leq f''(s_i ,t+1-\tau)z_i^2~.
    \end{cases}
\end{align*}

As a consequence,
\begin{align*}
    \la_U(\s+\z,t+1)-\la_U(\s,t) \leq& \frac{1}{N}\sum_{s_i>-R} f'(s_i ,t+1-\tau)z_i + \frac{1}{2} f''(s_i ,t+1-\tau)z_i^2 + \partial_tf(s_i, t+1-\tau)\\
    &+ \frac{1}{N}\sum_{s_i>-R}E_i~.\\
    \leq& \left(\sum_{s_i>-R} \partial_tf(s_i, t+1-\tau) + \frac{1}{2} f''(s_i ,t+1-\tau) \right) +\frac{1}{N}\sum_{s_i>-R}E_i\\
    \leq& \frac{1}{N}\sum_{s_i>-R}E_i~.
\end{align*}
The second inequality used the fact that
\begin{align*}
    \begin{cases}
        \sum_{s_i>-R} f'(s_i ,t+1-\tau)z_i \leq 0\\
        f''(s_i ,t+1-\tau)z_i^2 \leq f''(s_i ,t+1-\tau)~for~s_i > -R~;
    \end{cases}
\end{align*}
and the last inequality holds since $f$ satisfies Equation (\ref{eqn: equation for continuous action set 2}).

Using Lemma \ref{lma: Control of higher order derivatives of g tilda} and the fact that $|z_i|\leq 1,i\in[N]$ we get
\[
|\frac{1}{N}\sum_{s_i>-R}E_i| \leq \sup_{s_i > -R}|E_i| \leq \frac{C}{|t+1-\tau|^\frac{3}{2}}+\frac{C}{|t+1-\tau|^2}
\]

We now repeat the calculation done in Theorem \ref{thm:upper bound thm 2},
\begin{align*}
    \frac{1}{N}\sum_{i=1}^N L(s_{0,i})\leq\la_U(\s_0,0)&=\la_U(\bm{0}, -T)+\sum_{t=-T}^{-1}\la_U(\s_{t+1}, t+1)-\la_U(\s_t, t)\\
    &\leq\la_U(\bm{0}, -T)+C\sum_{t=-T-\tau}^{-\tau}\frac{1}{|t|^2}+\frac{1}{|t|^\frac{3}{2}}\\
    &\leq\la_U(\bm{0}, -T)+O(\frac{1}{\sqrt{\tau}}+\frac{1}{\tau})\\
    &=\la_U(\bm{0}, -T)+O(\frac{1}{\sqrt{\tau}}+\frac{1}{\tau})
\end{align*}
Note that
\begin{align*}
    \la_U(\bm{0}, -T)=f(0,-T-\tau)=\frac{2}{\sqrt{\pi}}\int_{\frac{R}{\sqrt{2(T+\tau)}}}^\infty e^{-x^2} dx
\end{align*}
We set $\tau=T^\theta$ where $\theta\in(0,1)$, then
\begin{align*}
    \frac{1}{N}\sum_{i=1}^N L(s_{0,i})\leq&
    \frac{2}{\sqrt{\pi}}\int^{\infty}_{\sqrt{\frac{\gamma T}{T+T^\theta}}}e^{-x^2}dx+O(T^{-\frac{\theta}{2}})\\
    =&\frac{2}{\sqrt{\pi}}\int^{\infty}_{\sqrt{\gamma}}e^{-x^2}dx+O(T^{-\frac{\theta}{2}})~.
\end{align*}
\end{proof}

Next we give the proof of lower bound
\begin{theorem}\label{thm: lower bound thm 3}
For any $T$ and $\theta\in(0,\frac{2}{3})$, if the number of chips $N \geq T^\frac{2+\theta}{4}$ then for any player strategy there exists an adversary strategy $\mathcal{A}$ associated with $\la_L$ (using $\beta=\tau=\lceil R + T^{\theta/2} \rceil - R$) such that
\begin{equation*}
    \frac{1}{N}\sum_{i=1}^N L(s_{0,i})\geq\frac{2}{\sqrt{\pi}}\int^{\infty}_{\sqrt{\gamma}}e^{-x^2}dx-O(T^{-\frac{\theta}{4}})~.
\end{equation*}
Moreover, $\mathcal{A}$ takes $z_i=0$ when chip $i$ is on the left of $-\lceil R + T^\frac{\theta}{2}\rceil$, and takes $|z_i|=1$ otherwise. 
\end{theorem}

\begin{proof}[Proof of Theorem \ref{thm: lower bound thm 3}]
We consider the adversary that only takes $\{-1,0,1\}$. With this choice of action set the chips always lie on integer points. Moreover, the adversary assigns $z_i=0$ whenever $s_i \leq - R - \beta$ and $|z_i|=1$ otherwise. 

We bound the increment of $\la_L$ as following. First note that when $s_i \leq - R - \beta$, the adversary chooses $z_i=0$, which implies that
\[
f(s_i+z_i+\beta,t+1-\tau)=f(s_i+\beta,t-\tau)=1~.
\]

As a consequence
\[
\la_L(\s+\z,t+1)-\la_L(\s,t)
    =\frac{1}{N}\sum_{s_i > - R - \beta} f(s_i+z_i+\beta , t+1-\tau) - f(s_i +\beta, t-\tau)~.
\]

Also note that $s_i$ is an integer and by our choice of $\beta$, $R+\beta=\lceil R+T^{\frac{\theta}{2}} \rceil$ is also an integer. Therefore when $s_i > - R - \beta$ we have $s_i + z_i \geq s_i - 1 \geq - R - \beta$. As a consequence $f(\cdot, t+1-\tau)$ is smooth between $s_i+\beta$ and $s_i+z_i+\beta$, and $f(s_i+\beta,\cdot)$ is smooth on $(t-\tau, t+1-\tau)$. We apply Taylor expansion in the case  $s_i > - R - \beta$
\begin{align*}
    &\frac{1}{N}\sum_{s_i > - R - \beta} f(s_i + z_i + \beta , t+1-\tau) - f(s_i + \beta, t-\tau)\\
    =&\frac{1}{N}\sum_{s_i > - R - \beta}\left(z_if'(s_i + \beta, t+1-\tau) + \partial_tf(s_i + \beta, t+1-\tau)+\frac{1}{2}f''(s_i + \beta, t+1-\tau)\right)\\
    &+\frac{1}{N}\sum_{s_i > - R - \beta}E_i\\
    =&\frac{1}{N}\sum_{s_i > - R - \beta}z_if'(s_i + \beta, t+1-\tau) + \frac{1}{N}\sum_{s_i > - R - \beta}E_i~,
\end{align*}
In the second equality we used the fact that $f$ satisfies Equation \ref{eqn: equation for continuous action set 2}.
The remainder term $E_i$ consists of $f^{(3)}(\cdot, t+1-\tau)$ and $\partial_{tt}f(s_i+\beta, \cdot)$. Using Lemma \ref{lma: Control of higher order derivatives of g tilda} and the fact that $|z_i| \leq 1$ there exists a constant $C$ such that
\[
|\frac{1}{N}\sum_{s_i > - R - \beta}E_i| \leq \sup_{s_i > - R - \beta}|E_i| \leq \frac{C}{|t+1-\tau|^\frac{3}{2}}+\frac{C}{|t+1-\tau|^2}~.
\]

To bound the first order term $\sum_{s_i > - R - \beta}z_if'(s_i+\beta, t+1-\tau)$ we apply Lemma \ref{lma: combinatoric lemma}. More specifically, in our case $a_i=-f'(s_i+\beta, t+1-\tau)$ and $U=\frac{C}{\sqrt{|t+1-\tau|}}$ by Lemma \ref{lma: Control of higher order derivatives of g tilda}. Lemma \ref{lma: combinatoric lemma} confirms there exists a subset $P\subset\{k:s_k > - R - \beta\}$ (note that we can always make the first inequality below holds by swapping $P$ and $\{k:s_k > - R - \beta\}\backslash P$) such that
\begin{align*}
    \begin{cases}
        \sum_{i\in P}p_i - \sum_{j\in \{k:s_k > - R - \beta\}\backslash P} p_j \geq 0\\
        |\sum_{i\in P}a_i - \sum_{j\in \{k:s_k > - R - \beta\}\backslash P}a_j| \leq U
    \end{cases}
\end{align*}
Thus by setting $z_i=1$ for $i\in P$ and $z_i=-1$ for $i\in\{j:s_j > - R - \beta\}\backslash P$, the adversary makes
\begin{align*}
    \begin{cases}
    \sum_{s_i > - R - \beta}z_i f'(s_i+\beta, t+1-\tau) \geq - \frac{C}{\sqrt{|t+1-\tau|}}\\
    \sum_{s_i > - R - \beta}p_i\cdot z_i \geq 0
    \end{cases}
\end{align*}
Moreover since $z_i=0$ for $s_i \leq - R - \beta$, we have
\begin{align*}
    \begin{cases}
    \sum_{i}z_i f'(s_i+\beta, t+1-\tau) \geq - \frac{C}{\sqrt{|t+1-\tau|}}\\
    \p\cdot\z \geq 0
    \end{cases}
\end{align*}

As a consequence, we can bound the final loss from below
\begin{align*}
    \frac{1}{N}\sum_{i=1}^N L(s_{0,i}) \geq \la_L(\s_0,0)&=\la_L(\bm{0}, -T)+\sum_{t=-T}^{-1}\la_L(\s_{t+1}, t+1)-\la_L(\s_t, t)\\
     &\geq \la_L(\bm{0}, -T)-C\sum_{t=-T+1-\tau}^{-\tau}\frac{1}{N|t|^{\frac{1}{2}}}+\frac{1}{|t|^\frac{3}{2}}+\frac{1}{|t|^2}\\
    &= \la_L(\bm{0}, -T) - O\left(\frac{\sqrt{T+\tau}}{N}+\frac{1}{\sqrt{\tau}}+\frac{1}{\tau}\right)~.
\end{align*}

Now we compute the main term
\begin{align*}
    \la_L(0, -T)=&f(\beta,-T-\tau) - f(-R+\beta, -\tau)\\
    =&\frac{2}{\sqrt{\pi}}\int^{\frac{\beta}{\sqrt{2\tau}}}_{\frac{R+\beta}{\sqrt{2(T+\tau)}}}e^{-x^2}dx\\
    \geq& \frac{2}{\sqrt{\pi}}\int^{\frac{\beta}{\sqrt{2\tau}}}_{\frac{R+\beta}{\sqrt{2T}}}e^{-x^2}dx~.
\end{align*}

Finally plugging in $\beta=\tau=\lceil R+T^{\frac{\theta}{2}}\rceil - R$ and $N=T^\frac{2+\theta}{4}$ for $\theta\in(0,\frac{2}{3})$ we get,
\begin{align*}
    \frac{1}{N}\sum_{i=1}^N L(s_{0,i})
    &\geq\frac{2}{\sqrt{\pi}}\int^{T^{\frac{\theta}{4}}/\sqrt{2}}_{\sqrt{\gamma}+\frac{T^\frac{\theta}{2}+1}{\sqrt{2T}}}e^{-x^2}dx-O(T^{-\frac{\theta}{4}})\\
    &=\frac{2}{\sqrt{\pi}}\int^{\infty}_{\sqrt{\gamma}}e^{-x^2}dx-O(T^{-\frac{\theta}{4}})~.
\end{align*}
\end{proof}

\section{Heuristic PDE Derivations}\label{Sec: heuristic pde derivations}

This paper has thus far considered a restricted class of drifting games, in which the moves $\z \in \R^N$
are restricted to a set of the form $\Z \times \cdots \times \Z$, where $\Z \subset \R$, and the final loss has the form $\frac{1}{N}\sum_{i=1}^N L(s_i)$. It is natural
to ask what becomes of our PDE-based approach when the set of permitted moves does not have this
structure. This section offers some thoughts in that direction.

Our method is to consider a \emph{scaled} version of the game, scaling moves by $\ep$ and time by $\ep^2$
where $\ep = 1/\sqrt{T}$, and to assume that the scaled minimax loss has a limit at $\ep \rightarrow 0$. This is
not a new idea: our scaling is the same as the one used in \cite{BrownBoost}, \cite{FREUND2002113},
and \cite{Boosting_book}.
It is also like the one used to study \emph{prediction with expert advice}
in \cite{drenska2019prediction}, where a heuristic calculation analogous to the
one in this section was given a mathematically rigorous justification.

We start, in Section \ref{subsec:7.1}, by introducing the scaled game; then, in
Section \ref{subsec:7.2}, we derive a nonlinear PDE
that (conjecturally) describes its limiting behavior, provided the set of permitted moves contains a neighborhood of the origin. When this last
hypothesis fails we do not find an asymptotic PDE, however we do find a PDE that the
player might reasonably use to determine a strategy; this is the focus of Section
\ref{subsec:7.3}.

\subsection{The Scaled Game} \label{subsec:7.1}

The drifting games we consider in this section are in most respects the same as introduced in the
Introduction. The only changes are:

\begin{itemize}
\item the set of possible moves $\Z_N$ is a bounded subset of $\R^N$ (where $N$ is the number
of chips); and
\item the final loss, denoted as $L_N(\s)$, is scaling invariant and is a decreasing function of each variable.
\end{itemize}

\noindent
Before scaling, the game's minimax loss is determined by the analogue of
Equation \eqref{eqn:DPP for original game}:
\begin{equation}\label{eqn:DPP for generalized game}
\la^d_\delta (\s,t)=\min_\p\max_{\z\in S_\delta(\p)}\la^d_\delta (\s+\z,t+1) \quad
\mbox{for $t \leq -1$} ,
\end{equation}
with the obvious extension of our previous notation
\begin{equation} \label{eqn:S-delta}
S_\delta (\p) = \{ \z\in {\Z}_N | \p \cdot \z \geq \delta\}
\end{equation}
and the final-time condition
\begin{equation} \label{eqn:final-time-condition-for-generalized-game}
\la^d_\delta(\s,0)= L_N(\s).
\end{equation}

Our goal is to understand the limiting behavior of $\la^d_\delta (\mathbf{0},-T)$ in the limit
when $T \rightarrow \infty$ and $\delta \rightarrow 0$ with $\delta = \sqrt{2 \gamma/T}$. With this in mind,
we set $\ep = 1/\sqrt{T}$ and introduce the scaled position and time variables
$$
\boldsigma = \ep \s, \quad \tau = \ep^2 t
$$
and the scaled minimax loss
\begin{equation}
\la^\ep (\boldsigma, \tau) = \la^d_\delta \left( \frac{\boldsigma}{\ep}, \frac{\tau}{\ep^2} \right).
\end{equation}
A moment's thought reveals that the dynamic programming principle defining $\la^d_\delta$ is
equivalent to
\begin{equation}\label{eqn:DPP for scaled game}
\la^\ep (\boldsigma,\tau)=\min_\p\max_{\z\in S_{\delta_\ep}(\p)}\la^\ep (\boldsigma +\ep\z,t+\ep^2)
\end{equation}
where $S_{\delta_\ep}$ is defined by \eqref{eqn:S-delta} with the parameter $\delta$ set equal to
\begin{equation} \label{eqn:delta-eps}
\delta_\ep = \sqrt{2\gamma} \ep
\end{equation}
and the final-time condition is
$$
\la^\ep (\boldsigma,0)= L_N(\boldsigma) .
$$
(We use here the \emph{scale-invariance} of $L_N$, i.e. the assumption that its
value at $\boldsigma\in\mathbb{R}^N$ is the same as its value at $\boldsigma/\ep$ for any $\ep > 0$). One can view $\la^\ep$
as the minimax loss of a scaled version of the drifting game, in which the permitted moves at a given
step are the vectors $\ep \z$ where $\z \in S_{\delta_\ep} (p)$. Note that the function $\la^\ep(\boldsigma, \tau)$
is defined when $\tau$ is a negative integer multiple of $\ep^2$, and understanding
$\la^d_\delta (\mathbf{0},-T)$
as $T \rightarrow \infty$ is equivalent to understanding $\lim_{\ep \rightarrow 0} \la^\ep (\mathbf{0},-1)$.

The preceding discussion used the hypothesis that $\delta = \sqrt{2 \gamma/T}$, which we
justified heuristically in Section \ref{subsec:2.2}. Let us offer here another argument why
$\delta_\ep$ should depend linearly on $\ep$. At the final time of the scaled game, the
final-time loss $L_N$ is evaluated at
$\ep(\z_{-1} + \z_{-1+\ep^2} + \ldots + \z_{-\ep^2})$. Since the adversary must choose $\z$ such that
$\p \cdot \z \geq \delta_\ep$ at each step, the bias introduced by $\delta_\ep$ at a single step is
of order $\ep \delta_\ep$ and this bias accumulates over $\ep^{-2}$ steps to
$\ep^{-2} \ep \delta_\ep = \ep^{-1} \delta_\ep$. For a nontrivial result in the limit
$\ep \rightarrow 0$, we evidently need $\delta_\ep$ to be linear in $\ep$. (Otherwise the accumulated bias
would dominate and the final-time loss function would be evaluated near $-\infty$ or near $+ \infty$.)
Since $\ep = 1/\sqrt{T}$, this justifies once again why $\delta$ must be proportional to
$1/\sqrt{T}$.

We remark that the minimax loss $\la^d_\delta$ and its scaled version $\la^\ep$ are non-increasing
functions of each ``spatial'' variable ($s_i$ for the former, $\sigma_i$ for the latter) at each time.
This is easily proved by backward induction in time, using the assumption that the final-time loss $L_N$ has this
property. Thus if $\la^\ep$ is differentiable then $\partial_i \la^\ep \leq 0$ for each $i$.

\subsection{The PDE Assuming \texorpdfstring{${\Z}_N$ Contains a Neighborhood of Origin}{interval}}\label{subsec:7.2}

We suppose now that the set of possible moves $\Z_N$ contains a neighborhood of the origin in $\R^N$.
This discussion generalizes what we did earlier in the paper for $\Z_N = [-1,+1]^N$. We shall Taylor-expand
the function $\la^\ep$, ignoring the possibility that it might not be smooth, and assuming that the
quantities we consider have limits as $\ep \rightarrow 0$. This is, of course, purely formal, however
analogous arguments are known to give correct results for many optimal control problems.

Substituting
\[
\la^\ep (\boldsigma+\ep\z,\tau+\ep^2) = \la^\ep (\boldsigma,\tau)+\ep\nabla\la^\ep(\boldsigma, \tau)\cdot\z +\ep^2(\partial_\tau\la^\ep(\boldsigma,\tau)+\frac{1}{2}\z^\top D^2\la^\ep(\boldsigma, \tau)\z)+O(\ep^3)
\]
into \eqref{eqn:DPP for scaled game} and dividing by $\ep$ gives
\begin{equation}\label{eqn:Taylor expansion of DPP}
0=\min_\p\max_{\z\in S_{\delta_\ep}(\p)}
\biggl(\nabla\la^\ep(\boldsigma, \tau) \cdot\z + \ep\bigl(\partial_\tau\la^\ep(\boldsigma,\tau)+
\frac{1}{2}\z^\top D^2\la^\ep(\boldsigma, \tau)\z\bigr)+O(\ep^2)\biggr).
\end{equation}

The leading order term is $\nabla\la^\ep(\boldsigma, \tau) \cdot\z$. Since $\ep$ is tending to $0$, this term
dominates both players' considerations. It is convenient to write $z_i = z_i' + \delta_\ep$, and to note
that $\z \in \Z_N$ is equivalent to $\z' \in {\Z}'_N = \Z_N - \delta_\ep \bm{1}$. Since
$$
\nabla\la^\ep \cdot\z = \delta_\ep \sum_{i=1}^N \partial_i \la^\ep + \nabla\la^\ep \cdot {\z}'
$$
and the first term on the right is independent of both $\p$ and $\z$, the leading-order min-max
reduces to
\begin{equation} \label{eqn:leading-order-min-max}
\min_\p \max_{\z' \in {\Z}_N', \, \p \cdot {\z}' \geq 0} \nabla \la^\ep(\boldsigma, \tau) \cdot {\z}'.
\end{equation}
We show in Appendix \ref{Appendix: PDE derivation} that the value of this
min-max is $0$, and it is achieved only when $\p$ is proportional to $-\nabla \la$ and
${\z}' \in  {\Z}'_N$ satisfies the additional condition $\z' \bot \nabla \la$.
The limiting PDE is therefore provided by the order-$\eps$ part of
\eqref{eqn:Taylor expansion of DPP}. Remembering that
$\delta_\ep = \sqrt{2\gamma} \ep$ and that ${\Z}'_N \rightarrow {\Z}_N$ as $\ep \rightarrow 0$, 
we conclude (heuristically) that $\lim_{\ep \rightarrow 0} \la$ should solve
\begin{align}\label{eqn: PDE for drifting game}
\begin{cases}
\partial_\tau\la(\boldsigma,\tau)+ \sqrt{2 \gamma} \sum_{i=1}^N \partial_i \la(\boldsigma,\tau) +
\frac{1}{2} \, \underset{\nabla\la(\boldsigma,\tau)\perp\z, \, \z\in\mathcal{Z}_N}{\max}\z^\top D^2\la(\boldsigma, \tau)\z=0\\
\la(\boldsigma,0)= L_N(\boldsigma)
\end{cases}
\end{align}
The first-order term $\sqrt{2 \gamma} \sum_{i=1}^N \partial_i \la(\boldsigma,\tau)$ can be eliminated by changing
variables from $(\boldsigma,\tau)$ to $({\boldsigma}', \tau)$ with $\sigma_i' = \sigma_i + \sqrt{2\gamma} \tau$.
The optimal $\z$ for
(\ref{eqn: PDE for drifting game}) cannot necessarily be used at finite $\ep$, since ${\Z}_N'$ is
slightly different from $\Z_N$. Thus our situation is slightly different from the prediction
with expert advice problem considered in \cite{drenska2019prediction}, where the asymptotically optimal adversary strategy is admissible at finite $\ep$. (We remark in passing that for small numbers of experts, asymptotically optimal strategies for \emph{prediction with expert advice} are in fact known explicitly
\cite{bayraktar19a, bayraktar19b, pmlr-v107-kobzar20a, pmlr-v125-kobzar20a}.)

The PDE (\ref{eqn: PDE for drifting game}) is highly nonlinear due to the maximization in $\z$.
When ${\Z}_N = [-1,1]^N$ it is natural to ask whether its solution has the form
$\la(\boldsigma, \tau) = \frac{1}{N} \sum_{i=1}^N f(\sigma_i,\tau)$ where $f$ solves
$\partial_\tau f + \sqrt{2 \gamma} f' + \frac{1}{2} \max \{f'', 0\} = 0$. The answer appears to be no:
to get this separable solution, one would need to replace the maximization over $\z$ in the
second-order term by $\underset{\z\in\mathcal{Z}_N}{\max}\z^\top D^2\la(\boldsigma, \tau)\z $ (changing the
equation, and therefore presumably its solution). Evidently: when $L_N(\boldsigma)=\frac{1}{N}\sum_{i=1}^N\ind{\sigma_i \leq 0}$, the present
discussion reduces to Equation \eqref{eqn: backward nonlinear equation} (up to change of variable) \emph{at best} in the
limit $N \rightarrow \infty$. (Lemma \ref{lma: combinatoric lemma} suggests that ignoring the constraint
$\z' \bot \nabla \la$ makes very little difference when $N$ is large enough).

\subsection{An Upper Bound Potential}\label{subsec:7.3}
When $\Z_N$ does not contain a neighborhood of the origin, one can begin as in the previous subsection,
but the optimal value of the leading-order min-max \eqref{eqn:leading-order-min-max} is unlikely to be
$0$. (The probabilistic argument used for our lower bound in Appendix \ref{Appendix: discrete action set} suggests
that it should approach $0$ in the limit as $N \rightarrow \infty$; however, to discuss an asymptotic PDE
we must hold the value of $N$ fixed.)

It is natural to ask whether our PDE-based approach can nevertheless be useful in this setting. We argue in
this subsection that it can be used to design a good potential for the player. The key point is that if
the player chooses $\p$ to be a multiple of $-\nabla \la^\ep$ then
\begin{equation} \label{eqn:key-point-section-7.3}
\max_{\z' \in {\Z}_N', \, \p \cdot {\z}' \geq 0} \nabla \la^\ep(\boldsigma, \tau) \cdot {\z}' \leq 0.
\end{equation}
While the optimal $\p$ might be better -- it might make the value of
\eqref{eqn:leading-order-min-max} negative -- the (heuristic) argument of the
previous subsection combined with \eqref{eqn:key-point-section-7.3} suggests that
$\la = \lim_{\ep \rightarrow 0} \la^\ep$ (if it exists) should satisfy

\begin{align}\label{ineq:subsolution inequality}
\begin{cases}
\partial_\tau \la(\boldsigma,\tau)+ \sqrt{2 \gamma} \sum_{i=1}^N \partial_i \la(\boldsigma,\tau) +
\frac{1}{2}\underset{\z \in\mathcal{Z}_N}{\max}\z^\top D^2\la(\boldsigma, \tau)\z \geq 0\\
\la(\boldsigma,0)= L_N(\boldsigma) .
\end{cases}
\end{align}

This insight can be used by the player as follows: a function $\la_U(\s, t)$ satisfying the opposite inequality 
\begin{align}\label{eqn:equation for upper bound potential}
\begin{cases}
\partial_\tau\la(\boldsigma,\tau)+ \sqrt{2 \gamma} \sum_{i=1}^N \partial_i \la(\boldsigma,\tau)
\frac{1}{2}\max_{\z\in {\Z}_N} \z^\top D^2\la(\boldsigma, \tau)\z \leq 0\\
\la(\boldsigma,0)= L_N(\boldsigma) .
\end{cases}
\end{align}
provides a good player potential. In particular, our upper-bound arguments seem to apply (at least formally) for such $\la_U$; moreover, the comparison principle (which holds for such parabolic PDEs) shows that solution of \eqref{ineq:subsolution inequality} and \eqref{eqn:equation for upper bound potential} must satisfy $\la(0, -T) \leq \la_U(0, -T)$.

The best upper bound (the smallest $\la_U$) should solve \eqref{eqn:equation for upper bound potential} with the inequality replaced by equality. This PDE is nonlinear, in general since it involves a maximization over $\z$. However in the separable case $z_N = \{ \pm 1\}^N$ it is easy to see that
$\la(\boldsigma,\tau) = \frac{1}{N}\sum_{i=1}^N f(\sigma_i,\tau)$ where $f$ solves $\partial_\tau f + \sqrt{2 \gamma} f' +
\frac{1}{2} f'' = 0$. When the first-order term is eliminated by the change of
variables $\sigma_i' = \sigma_i + \sqrt{2 \gamma}\tau$ and $L_N(\boldsigma) = \frac{1}{N}\sum_{i=1}^N\ind{\sigma_i \leq 0}$, this reduces to the linear heat equation whose solution
we used to design our potentials in Section \ref{Sec: discrete action set}.

Our analysis of the separable case ${\Z}_N = \{ \pm 1 \}^N$ in Section \ref{Sec: discrete action set}
used a probabilistic argument to see that the leading-order min-max
\eqref{eqn:leading-order-min-max} is very close to $0$ when $N$ is sufficiently large. While that discussion was
limited to ${\Z}_N = \{ \pm 1 \}^N$, we suppose a similar argument could be used
for other choices of ${\Z}_N$. 

\section{The leading-order min-max in Section \ref{subsec:7.2}}
\label{Appendix: PDE derivation}

We want to show that for any nonzero $\boldxi \in \R^N$ with non-positive components,
and any bounded $A \subset \R^N$ containing
a neighborhood of the origin,
\begin{equation} \label{eqn:goal-of-appendixG}
\min_{\p \in \Delta_N} \max_{\z \in A, \, \p \cdot \z \geq 0} \boldxi \cdot \z = 0,
\end{equation}
and this value is achieved only when $\p = - \boldxi/\| \boldxi \|_1$ and $\p \cdot \z = 0$.
(This assertion was used in Section \ref{subsec:7.2} with
$A = {\Z}_N - \delta_\ep {\mathbf 1}$ and $\boldxi = \nabla \la^\ep$; see Equation
\eqref{eqn:leading-order-min-max} and the text just after it.)

We first prove the following geometric lemma.
\begin{lemma}\label{lma:geometric lemma}
Suppose $\mathbf{a}, \mathbf{b}\in\mathbb{R}^N$ are non-zero vectors and only have non-negative
components, moreover if they are not parallel, then there exists a vector $\mathbf{v}$ such that
$\mathbf{a}\cdot\mathbf{v}>0$ and $\mathbf{b}\cdot\mathbf{v}<0$.
\end{lemma}
\begin{proof}[Proof of lemma \ref{lma:geometric lemma}]
We assume $\mathbf{v}=\mu\mathbf{a}-\mathbf{b}$, $\mu>0$. To satisfy $\mathbf{a}\cdot\mathbf{v}>0$ and
$\mathbf{b}\cdot\mathbf{v}<0$, $\mu$ must be such that
\begin{align*}
\begin{cases}
\mathbf{a}\cdot\mathbf{b}<\mu\|a\|_2^2\\
\mathbf{a}\cdot\mathbf{b}<\frac{1}{\mu}\|b\|_2^2.
\end{cases}
\end{align*}
If $\mathbf{a}\cdot\mathbf{b}=0$ then the above inequalities hold for any $\mu>0$.
For the case of $\mathbf{a}\cdot\mathbf{b}>0$, since $\mathbf{a}$ and $\mathbf{b}$ are not parallel,
\[
(\mathbf{a}\cdot\mathbf{b})^2<\|a\|_2^2\|b\|_2^2.
\]
Set $\mu=\mu_0=\frac{\mathbf{a}\cdot\mathbf{b}}{\|a\|_2^2}>0$, we have
\begin{align*}
\begin{cases}
\mathbf{a}\cdot\mathbf{b}=\mu\|a\|_2^2\\
\mathbf{a}\cdot\mathbf{b}<\frac{1}{\mu}\|b\|_2^2.
\end{cases}
\end{align*}
Thus setting $\mu$ to be slightly larger than $\mu_0$ will meet the constraints.
\end{proof}

Turning now to \eqref{eqn:goal-of-appendixG}, consider first what happens if
$\p\in\Delta_N$ and $-\boldxi$ are not parallel. Then by the Lemma,
there exists a vector $\v$ such that
\begin{align*}
    \begin{cases}
    \p\cdot\v>0\\
    \boldxi \cdot \v > 0.
    \end{cases}
\end{align*}
Replacing $\v$ by $\lambda \v$ for $\lambda > 0$ leaves the conclusion unchanged. Since $A$
contains a neighborhood of the origin, we conclude if $\p$ and $-\boldxi $ are not parallel, then
\[
\max_{\z \in A, \, \p \cdot \z \geq 0} \nabla \boldxi \cdot \z > 0.
\]
On the other hand, if $\p$ is parallel to $-\boldxi$, i.e.
$\p=-\boldxi/ \| \boldxi \|$, then it's clear that
\[
\max_{\z \in A, \, \p \cdot \z \geq 0} \nabla \boldxi \cdot \z = 0,
\]
and equality is obtained exactly when $\p\cdot \z = 0$. (There actually exists such $\z$, since
by $A$ contains a neighborhood of the origin.) This completes the verification of our assertion.

\end{document}